\documentclass[11pt]{article} 

\usepackage[left=1in,top=1in,right=1in,bottom=1in,letterpaper]{geometry}

\usepackage{microtype}
\usepackage{graphicx}
\usepackage{subfigure}
\usepackage{booktabs} 
\usepackage{hyperref}

\usepackage{amsmath}
\usepackage{mathrsfs}
\usepackage{amssymb}
\usepackage{algorithm}
\usepackage{algorithmic}

\usepackage[numbers,sort&compress]{natbib}
\usepackage{color}
\usepackage{enumerate}

\newtheorem{definition}{Definition}[section]
\newtheorem{assumption}{Assumption}
\newtheorem{proposition}{Proposition}[section]
\newtheorem{theorem}{Theorem}
\newtheorem{remark}[theorem]{Remark}
\newtheorem{lemma}{Lemma}[section]

\newcommand{\argmax}{\mathop{\rm argmax}}

\newcommand{\bE}{\textsf{\text{E}}}
\newcommand{\cY}{\mathcal{Y}}
\newcommand{\br}{\mathbb{R}}
\newcommand{\bu}{\boldsymbol{u}}
\newcommand{\cM}{\mathcal{M}}
\newcommand{\cO}{\mathcal{O}}
\newcommand{\Proj}{\mathrm{Proj}}
\newcommand{\bX}{\boldsymbol{X}}
\newcommand{\bx}{\boldsymbol{x}}
\newcommand{\cG}{\mathcal{G}}
\newcommand{\bka}{\bar{\kappa}}

\newcommand{\be}{\begin{equation}}
\newcommand{\ee}{\end{equation}}
\newcommand{\ba}{\begin{array}}
\newcommand{\ea}{\end{array}}
\newcommand{\bad}{\begin{aligned}}
\newcommand{\ead}{\end{aligned}}

\allowdisplaybreaks

\title{Zeroth-Order Algorithms for Nonconvex-Strongly-Concave Minimax Problems with Improved Complexities}
\author{Zhongruo Wang\thanks{Department of Mathematics, University of California, Davis.  \texttt{zrnwang@ucdavis.edu}} 
\and Krishnakumar Balasubramanian\thanks{Department of Statistics, University of California, Davis. \texttt{kbala@ucdavis.edu}}
\and Shiqian Ma\thanks{Department of Mathematics, University of California, Davis. \texttt{sqma@ucdavis.edu.}}
\and Meisam Razaviyayn\thanks{Department of Industrial and Systems Engineering, University of Southern California. \texttt{razaviya@usc.edu.}}
}

\begin{document}

\maketitle

\begin{abstract}
In this paper, we study zeroth-order algorithms for minimax optimization problems that are nonconvex in one variable and strongly-concave in the other variable. Such minimax optimization problems have attracted significant attention lately due to their applications in modern machine learning tasks. We first consider a deterministic version of the problem. We design and analyze the Zeroth-Order Gradient Descent Ascent (\texttt{ZO-GDA}) algorithm, and provide improved results compared to existing works, in terms of oracle complexity. We also propose the Zeroth-Order Gradient Descent Multi-Step Ascent (\texttt{ZO-GDMSA}) algorithm that significantly improves the oracle complexity of \texttt{ZO-GDA}. We then consider stochastic versions of \texttt{ZO-GDA} and \texttt{ZO-GDMSA}, to handle stochastic nonconvex minimax problems. For this case, we provide oracle complexity results under two assumptions on the stochastic gradient: (i) the uniformly bounded variance assumption, which is common in traditional stochastic optimization, and (ii) the Strong Growth Condition (SGC), which has been known to be satisfied by modern over-parametrized machine learning models. We establish that under the SGC assumption, the complexities of the stochastic algorithms match that of deterministic algorithms. Numerical experiments are presented to support our theoretical results. 
\end{abstract}

\section{Introduction}

Algorithms for solving optimization problems with only access to (noisy) evaluations of the objective function are called zeroth-order algorithms. These algorithms have been studied for decades in the optimization literature; see, for example, \cite{conn2009introduction,rios2013derivative,audet2017derivative} for a detailed overview of the existing approaches. Recently, the study of zeroth-order optimization algorithms has gained significant attention also in the machine learning literature, due to several motivating applications, for example, in designing black-box attacks to deep neural networks~\cite{chen2017zoo}, hyperparameter tuning~\cite{snoek2012practical}, reinforcement learning~\cite{moriarty1999evolutionary,salimans2017evolution} and bandit convex optimization~\cite{bubeck2017kernel}. However, a majority of the zeroth-order optimization algorithms in the literature has been developed for minimization problems. 

In this work, we study zeroth-order optimization algorithms for solving nonconvex minimax problems (aka saddle-point problems). Specifically, we consider both the deterministic setting:
\begin{align}\label{min_max_problem_deterministic}
    \min_{x\in \mathbb{R}^{d_1}} \max_{y\in \mathcal{Y}} f(x,y),
\end{align}
and the stochastic setting:
\begin{align}\label{min_max_problem_online}
    \min_{x\in \mathbb{R}^{d_1}} \max_{y\in \mathcal{Y}} f (x,y) = \bE_{\xi \sim \mathcal{P}}F(x,y,\xi).
\end{align}
Here, $F(x,y,\xi)$ and hence $f(x,y)$ are assumed to be sufficiently smooth functions, $\mathcal{Y} \subset \mathbb{R}^{d_2}$ is a closed and convex constraint set\footnote{One of our algorithms works also in the unconstrained setting. See Remark~\ref{remark:cversusuc} for more details.}, 
and $\mathcal{P}$ is a distribution characterizing the stochasticity in the problem. We allow for the function $f(\cdot, y)$ to be nonconvex for all $y \in \mathbb{R}^{d_2}$ but require $f(x,\cdot)$ to be strongly-concave for all $x \in \mathbb{R}^{d_1}$. 
One of our main motivations for studying zeroth-order algorithms for nonconvex minimax problems is their application in designing black-box attacks to deep neural networks. By now, it is well established that care must be taken when designing and training deep neural networks as it is possible to design adversarial examples that would make the deep network to misclassify, easily. Since the intriguing works of~\cite{szegedy2013intriguing,Liu2016DelvingIT}, the problem of designing such adversarial examples that transfer across multiple deep neural networks models has also been studied extensively. As the model architecture is unknown to the adversary, the problem could naturally be formulated to solve a minimax optimization problem under the availability of only (noisy) objective function evaluation. We refer the reader to~\cite{liu2019min} for details regarding such formulations. Apart from the above applications, we also note that zeroth-order minimax optimization problems also arise in multi-agent reinforcement learning with bandit feedback~\cite{wei2017online, zhang2019multi}, robotics~\cite{wang2017max,bogunovic2018adversarially} and distributionally robust optimization~\cite{namkoong2016stochastic}.

Recently, there has been an ever-growing interest in analyzing first-order algorithms for the case of nonconvex-concave and nonconvex-nonconcave minimax problems, motivated by its applications to training generative adversarial networks~\cite{goodfellow2014generative}, AUC maximization~\cite{ying2016stochastic}, designing fair classifiers~\cite{agarwal2018reductions}, robust learning systems~\cite{madry2017towards} fair machine learning \cite{Zhang-mitigating-2018,Xu-Fairgan-2018,Mesiam-Renyi-2019}, and reinforcement learning~\cite{pfau2016connecting, dai2017sbeed, neyman2003stochastic, filar2012competitive}. Specifically,~\cite{lu2019hybrid, rafique2018non,nouiehed2019solving, sanjabi2018convergence, lin2019gradient,thekumparampil2019efficient} proposed and analyzed variants of gradient descent ascent for nonconvex-concave objectives. Very recently, under a stronger mean-squared Lipschitz gradient assumption~\cite{luo2020stochastic} obtained improved complexities for stochastic nonconvex-concave objectives. Furthermore,~\cite{daskalakis2017training, daskalakis2018limit, hsieh2018finding, mertikopoulos2018cycles, piliouras2018learning, gidel2018variational, oliehoek2018beyond, jin2019minmax, flokas2019poincar} studied general nonconvex-nonconcave objectives. 

Compared to first-order algorithms, zeroth-order algorithms for minimax optimization problems are underdeveloped. Motivated by the need for robustness in optimization, \cite{Menickelly2018DerivativefreeRO} proposed derivative-free algorithms for saddle-point optimization. However, they do not provide non-asymptotic oracle complexity analysis. Bayesian optimization algorithms and evolutionary algorithms were proposed in~\cite{bogunovic2018adversarially, picheny2019bayesian} and~\cite{bertsimas2010robust, al2018application} respectively for minimax optimization, targeting robust optimization and learning applications. The above works do not provide any oracle complexity analysis. Recently,~\cite{roy2019online} studied zeroth-order Frank-Wolfe algorithms for strongly-convex and strongly-concave constrained saddle-point optimization problems and provided non-asymptotic oracle complexity analysis. Furthermore,~\cite{liu2019min} studied zeroth-order algorithms for nonconvex-concave minimax problems, similar to our setting. More recently, \cite{Anagnostidis-2021} proposed a stochastic direct search method for \eqref{min_max_problem_online} under the assumption of the Polyak-\L{}ojasiewicz (PL) condition. \cite{Xu-2021} and \cite{Huang-2020} also studied zeorth-order methods for \eqref{min_max_problem_online}, where they required mean-squared smoothnesss assumption, which is stronger than our assumptions.

{\bf Our Contributions.} 
In this work, we consider both deterministic and stochastic minimax problems in the form of \eqref{min_max_problem_deterministic} and \eqref{min_max_problem_online}, respectively. A detailed comparison of our algorithms and existing methods is given in Table \ref{tab:comparison}. Our contributions lie in several folds. 
\begin{enumerate}[(i)]
\item For deterministic minimax problem \eqref{min_max_problem_deterministic}, we design a zeroth-order gradient descent ascent (\texttt{ZO-GDA}) algorithm, whose oracle complexity improves the currently best known one in \cite{liu2019min} under the same assumptions. Notably, for this algorithm, the set $\mathcal{Y}$ could be constrained or unconstrained (i.e., the entire Euclidean space $\mathbb{R}^{d_2}$).
\item For deterministic minimax problem \eqref{min_max_problem_deterministic}, we propose a novel zeroth-order gradient descent multi-step ascent (\texttt{ZO-GDMSA}) algorithm, which is motivated by \cite{nouiehed2019solving}. This algorithm performs multiple steps of gradient ascent followed by one single step of gradient descent in each iteration. Its oracle complexity is significantly better than that of \texttt{ZO-GDA} in terms of the condition number dependency. To the best of our knowledge, this is the best complexity result for zeroth-order algorithms for solving deterministic minimax problems so far under the assumptions in Section \ref{sec:assumptions}.
\item We desgin and analyze the stochastic counterparts of \texttt{ZO-GDA} and \texttt{ZO-GDMSA} and establish their oracle complexity under two settings: (i) uniformly bounded variance assumption on the stochastic gradient, which is standard in stochastic optimization, and (ii) the Strong Growth Condition (SGC)~\cite{vaswani2019fast}, which is satisfied by modern overparametrized machine learning models. Notably, under SGC, we show that the complexities of the stochastic algorithms are the same as their deterministic counterparts.  
\end{enumerate}

\begin{table*}[t!]
\centering
\resizebox{\textwidth}{!}{
\begin{tabular}{lllll} 
Algorithm & Order& Complexity & Objective function & Constraint $(x,y)$ \\\hline
GDmax (\cite{lin2019gradient})        &      1st                &                    $\cO(\kappa^2 \epsilon^{-2})$& NC-SC &      U,C   \\
SGDmax (\cite{lin2019gradient})        &      1st                &                    $\cO(\kappa^3 (\sigma_1^2+\sigma_2^2)\epsilon^{-4})$& NC-SC &      U,C   \\
Multi-step GDA(\cite{nouiehed2019solving})  &                      1st & $\tilde{\mathcal{O}}(\log(\epsilon^{-1})\epsilon^{-2})$         & NC-PL        &   C,U   \\
Multi-step GDA (\cite{nouiehed2019solving})          &      1st                &                    $\tilde{\mathcal{O}}(\log(\epsilon^{-1})\epsilon^{-3.5})$& NC-C &      C,C  \\ \hline
  ZO-min-max(\cite{liu2019min})   & 0th                      & $\tilde{\mathcal{O}}((d_1+d_2)\epsilon^{-6})$ & NC-SC & C,C        \\ 
\texttt{ZO-GDA}           & 0th        & {$\cO(\kappa^5(d_1+d_2)\epsilon^{-2})$}      & NC-SC     &U,C or U    \\
 \texttt{ZO-GDMSA}      & 0th      & $\cO(\kappa(d_1 + \kappa d_2 \log(\epsilon^{-1}))\epsilon^{-2})$ & NC-SC    & U,C \\
  \texttt{ZO-SGDA} (Assmp.~\ref{assumption:stoch})   & 0th   & {$\cO(\kappa^5(\sigma_1^2d_1+\sigma_2^2d_2)\epsilon^{-4})$} & NC-SC   &U,C or U \\
 \texttt{ZO-SGDMSA} (Assmp.~\ref{assumption:stoch})  & 0th   & $\cO(\kappa(d_1\sigma_1^2  +\kappa d_2\sigma_2^2\log(\epsilon^{-1}))\epsilon^{-4})$ & NC-SC &U,C        \\
   \texttt{ZO-SGDA} (Assmp.~\ref{SGC-assumption})   & 0th  & {$\cO(\kappa^5(\rho_1d_1+\rho_2d_2)\epsilon^{-2})$}  & NC-SC  &U,C or U \\
 \texttt{ZO-SGDMSA} (Assmp.~\ref{SGC-assumption})  & 0th & $\cO(\kappa(d_1\rho_1  +\kappa d_2\rho_2\log(\epsilon^{-1}))\epsilon^{-2})$ & NC-SC &U,C        \\\hline
\end{tabular}
}
\caption{Comparison of different algorithms. The column of ``Complexity'' gives the complexity of calls to zeroth-order oracle or first-order oracle. Here, we use $\tilde{\mathcal{O}}$ to hide the $\kappa$ dependency ({where $\kappa$ refers to the condition number as defined in Assumption~\ref{assumption:A1-A2-A4}}), as it was not explictly tracked and stated in \cite{liu2019min,nouiehed2019solving}. In the column of ``Objective function'' column, ``NC-SC'' indicates that the objective function is nonconvex with respect to $x$ and strongly concave with respect to $y$. ``C'' indicates that the function is concave with respect to $y$. PL denotes the Polyak-\L{}ojasiewicz condition.  In the column of ``Constraint'', ``C'' means ``constrained'' and ``U'' means ``unconstrained''.}
\label{tab:comparison}
\end{table*}

The rest of this paper is organized as follows. In Section \ref{sec:assumptions} we provide some preliminaries and introduce our zeroth-order gradient estimator. In Section \ref{sec:deterministic} we present our \texttt{ZO-GDA} and \texttt{ZO-GDMSA} for solving the deterministic minimax problem \eqref{min_max_problem_deterministic}, and analyze their oracle complexities. In Section \ref{sec:stochasticsetting} we present the stochastic algorithms \texttt{ZO-SGDA} and \texttt{ZO-SGDMSA} for solving the stochastic minimax problem \eqref{min_max_problem_online}, and analyze their oracle complexities. In Section \ref{sec:numerical} we provide some numerical results to our stochastic algorithms \texttt{ZO-SGDA} and \texttt{ZO-SGDMSA} for solving a distributionally robust optimization problem. We draw some conclusions in Section \ref{discussion}. Proofs of all theorems are provided in the appendix. 

\section{Preliminaries} \label{sec:assumptions}

Assumption \ref{assumption:A1-A2-A4} is made throughout the paper.

\begin{assumption}\label{assumption:A1-A2-A4}
The objective function $f(x,y)$ and the constraint set $\cY$ have the following properties:
\begin{itemize} 
\item $f(x,y)$ is continuously differentiable in $x$ and $y$, and $f(\cdot, y)$ {could be potentially} nonconvex for all $y \in \mathcal{Y}$ and $f(x,\cdot)$ is $\tau$-strongly concave for all $x \in \mathbb{R}^{d_1}$.
\item When viewed as a function in $\br^{d_1+d_2}$, $f(x,y)$ is $\ell$-gradient Lipschitz. That is, there exists constant $\ell>0$ such that $\forall x_1,x_2\in\br^{d_1}, y_1,y_2\in\cY.$
        \be\label{assumption-Lips}
        \|\nabla f(x_1,y_1) - \nabla f(x_2,y_2)\|_2 \leq \ell \|(x_1,y_1)-(x_2,y_2)\|_2, 
        \ee
We use $\kappa := \ell/\tau$ to denote the problem condition number throughout this paper.
\item The function $g(x):=\max_{y\in\cY} f(x,y)$ is lower bounded. {Moreover, we assume that function $g$ is $L_g$-smooth, i.e., $\|\nabla g(x) - \nabla g(y)\| \leq L_g \|x-y\|_2$, for all $x,y\in\mathbb{R}^{d_1}$. As will be shown later in Lemma \ref{lemma:smoothed convexity}, this is indeed true with $L_g = (1+\kappa)\ell$.} 
\item The constraint set $\mathcal{Y} \subset \mathbb{R}^{d_2}$ is bounded and convex, with diameter $D >0$. The boundedness assumption can be relaxed (see Remark \ref{remark:cversusuc}). 
\end{itemize}
\end{assumption}




The following assumption, which is standard in the literature \cite{nesterov2017random, Ghadimi_2013, 2018arXiv180906474B}, will also be used in our paper. 
\begin{assumption}[Uniformly Bounded Variance]\label{assumption:stoch}
For any $x \in \mathbb{R}^{d_1}$ and $y \in \cY$, the stochastic zeroth-order oracle outputs an estimator $F\left(x,y,\xi\right)$ of $f\left(x,y\right)$ such that $\bE_\xi[F\left(x,y,\xi\right)]=f\left(x,y\right)$ and $
\bE_\xi[\nabla_{x} F\left(x,y,\xi\right)]=\nabla_{x} f\left(x,y\right)$, $ \bE_\xi[\nabla_{y} F\left(x,y,\xi\right)]=\nabla_{y} f\left(x,y\right)$,
$ \bE_{\xi}(\|\nabla_x F(x,y,\xi)  - \nabla_x f(x,y)\|_2^2) \leq \sigma_1^2$, and $ \bE_{\xi}(\| \nabla_y F(x,y,\xi) -\nabla_y f(x,y) \|_2^2) \leq \sigma_2^2$.
\end{assumption}

In addition to Assumptions \ref{assumption:A1-A2-A4} and \ref{assumption:stoch}, motivated by over-parametrized models arising in modern machine learning problems \cite{vaswani2019fast}, we also consider the following SGC assumption on the stochastic gradient.

\begin{assumption}[Strong Growth Condition~\cite{vaswani2019fast}]\label{SGC-assumption}
There exist $\rho_1,\rho_2 > 1$ such that the following is true for the stochastic gradients: 
\[\bE_\xi (\|\nabla_x F(x,y,\xi)\|_2^2) \leq \rho_1 \|\nabla_x f(x,y)\|_2^2, \mbox{ and } \bE_\xi (\|\nabla_y F(x,y,\xi)\|_2^2) \leq \rho_2 \|\nabla_y f(x,y)\|_2^2.\]
\end{assumption}

This condition is widely observed to be satisfied in modern over-parametrized models (e.g., deep neural networks) and has been used extensively for minimization problems recently \cite{ma2018power, bassily2018exponential, meng2019fast, vaswani2019painless, roy2020escaping}. 

\subsection{Zeroth-order gradient estimator} 

We now discuss the idea of zeroth-order gradient estimator based on Gaussian smoothing technique \cite{nesterov2017random}. For the deterministic case, we denote $\bu_1 \sim N(0,\textbf{1}_{d_1})$, $\bu_2 \sim N(0,\textbf{1}_{d_2})$, where $\textbf{1}_{d_1}$ and $\textbf{1}_{d_2}$ denote identity matrices with sizes $d_1\times d_1$ and $d_2\times d_2$, respectively. The notion of the Gaussian smoothed functions is defined as follows:
\begin{align}\label{def-f-mu}
f_{\mu_1}(x,y):&=\bE_{\bu_1}f(x+\mu_1 \bu_1,y), \\ \nonumber
f_{\mu_2}(x,y):&=\bE_{\bu_2}f(x,y+\mu_2 \bu_2),
\end{align}
and the zeroth-order gradient estimators \cite{nesterov2017random} are defined as
\begin{align}\label{def-G-mu-H-mu}
G_{\mu_1}(x,y,\bu_1) &= \frac{f(x+\mu_1\bu_1,y) - f(x,y)}{\mu_1}\bu_1, \\ \nonumber
 H_{\mu_2}(x,y,\bu_2) &= \frac{f(x,y+\mu_2\bu_2) - f(x,y)}{\mu_2}\bu_2,
\end{align}
where $\mu_1>0$ and $\mu_2>0$ are smoothing parameters.

{As noted in~\cite{balasubramanian2018zeroth}, the Gaussian smoothing technique proposed by~\cite{nesterov2017random} is based on the Stein's identity~\cite{stein1972bound}, for characterizing Gaussian random vectors. Specifically, Stein's identity states that a random vector $u \in \mathbb{R}^d$, is standard Gaussian \emph{if and only if}, $\bE\left[u~h(u)\right] = \bE\left[\nabla h(u)\right]$, for all absolutely continuous functions $h:\mathbb{R}^d\to\mathbb{R}$. Note that Stein's identity, naturally relates function queries to gradients and thus is naturally suited for zeroth-order optimization. If we let $h(u)$ to be the Gaussian smoothed functions (as in~\eqref{def-f-mu}), it is easy to see that the zeroth-order gradients (as in~\eqref{def-G-mu-H-mu}) follow by simply evaluating the Gaussian Stein's identity.}

It should be noted following the arguments in~\cite{nesterov2017random, 2018arXiv180906474B} that
$\bE_{\bu_1} G_{\mu_1}(x,y,\bu_1) = \nabla_x f_{\mu_1}(x,y)$, and $\bE_{\bu_2} G_{\mu_2}(x,y,\bu_2) = \nabla_y f_{\mu_2}(x,y)$. Hence, the zeroth-order gradient estimators in \eqref{def-G-mu-H-mu} provide unbiased estimates of the gradient of Gaussian smoothed functions $f_{\mu_1}(x,y,\bu_1):=f(x+\mu_1 \bu_1,y)$ and $f_{\mu_2}(x,y,\bu_2):=f(x,y+\mu_2 \bu_2)$. Similarly, for the stochastic case, the Gaussian smoothed functions are defined as:
\begin{align}\label{def-f-mu-sto}
f_{\mu_1}(x,y):=\bE_{\bu_1,\xi}F(x+\mu_1 \bu_1,y,\xi),\\ \nonumber
 f_{\mu_2}(x,y):=\bE_{\bu_2,\xi}F(x,y+\mu_2 \bu_2,\xi),
\end{align}
and the zeroth-order stochastic gradient estimators are defined as:
\begin{align}\label{def-G-mu-H-mu-sto}
G_{\mu_1}(x,y,\bu_1,\xi) = \frac{F(x+\mu_1\bu_1,y,\xi) - F(x,y,\xi)}{\mu_1}\bu_1,\\ \nonumber
 H_{\mu_2}(x,y,\bu_2,\xi) = \frac{F(x,y+\mu_2\bu_2,\xi) - F(x,y,\xi)}{\mu_2}\bu_2.
\end{align}
One can also show that the zeroth-order gradient estimators provide unbiased estimates to the gradients of the Gaussian smoothed functions, i.e., $\bE_{\bu_1,\xi} G_{\mu_1}(x,y,\bu_1,\xi) = \nabla_x f_{\mu_1}(x,y),$ and $\bE_{\bu_2,\xi} H_{\mu_2}(x,y,\bu_2,\xi) = \nabla_y f_{\mu_2}(x,y).$

In our algorithms, we also need to use mini-batch zeroth-order gradient estimators, which can reduce the variance of stochastic gradient estimators. To this end, we define the following notation. For integer $q>0$, we denote $[q]:=\{1,\ldots,q\}$. In the deterministic case, for integers $q_1>0$, $q_2>0$ we denote
\begin{align}\label{def-G-mu-H-mu-q}
G_{\mu_1}(x,y,\bu_{1,[q_1]})=\frac{1}{q_1}\sum_{i=1}^{q_1} G_{\mu_1}(x,y,\bu_{1,i}),\\ \nonumber
H_{\mu_2}(x,y,\bu_{2,[q_2]})=\frac{1}{q_2}\sum_{i=1}^{q_2} H_{\mu_2}(x,y,\bu_{2,i}).
\end{align}
For indices sets $\mathcal{M}_1$ and $\mathcal{M}_2$, in the stochastic case we denote
\be\label{def-G-mu-H-mu-sto-M}
\ba{ll}
G_{\mu_1}(x,y,\bu_{\cM_1},\xi_{\mathcal{M}_1})&=\frac{1}{|\mathcal{M}_1|}\sum_{i\in\mathcal{M}_1}G_{\mu_1}(x,y,\bu_{1,i},\xi_i),\\
H_{\mu_2}(x,y,\bu_{\cM_2},\xi_{\mathcal{M}_2})&=\frac{1}{|\mathcal{M}_2|}\sum_{i\in\mathcal{M}_2}H_{\mu_2}(x,y,\bu_{2,i},\xi_i).
\ea\ee
It is easy to see that we have the following unbiasedness properties:
\[
\bE_{\bu_{1,[q_1]}} G_{\mu_1}(x,y,\bu_{1,[q_1]}) = \nabla_x f_{\mu_1}(x,y) \mbox{ and } \bE_{\bu_{2,[q_2]}} H_{\mu_2}(x,y,\bu_{2,[q_2]}) = \nabla_y f_{\mu_2}(x,y)
\]
and
\begin{align*}
\bE_{\bu_1}\bE_{\xi_{\cM_1}} G_{\mu_1}(x,y,\bu_{\cM_1},\xi_{\cM_1}) = \nabla_x f_{\mu_1}(x,y) \\ 
\bE_{\bu_2}\bE_{\xi_{\cM_2}} H_{\mu_2}(x,y,\bu_{\cM_2},\xi_{\cM_2}) = \nabla_y f_{\mu_2}(x,y).
\end{align*}


\subsection{Complexity Measure}

Following~\cite{lin2019gradient}, the $\epsilon$-stationary point of problems \eqref{min_max_problem_deterministic} and \eqref{min_max_problem_online} and is defined as follows. 

\begin{definition}\label{saddleopt}
A point $(\bar{x},\bar{y})$ is called an $\epsilon$-stationary point of problem \eqref{min_max_problem_deterministic} and \eqref{min_max_problem_online} if it satisfies the following conditions: $\bE(\|\nabla_ x f(\bar{x},\bar{y})\|^2_2)\leq \epsilon^2$ and $\bE(\|\nabla_ y f(\bar{x},\bar{y})\|^2_2)\leq \epsilon^2$. {Here, the expectation is over $\textbf{u}_1$ and $\textbf{u}_2$ sequence for problem \eqref{min_max_problem_deterministic}, and over the $\textbf{u}_1$, $\textbf{u}_2$ and $\xi$ sequence for problem \eqref{min_max_problem_online}. The $\textbf{u}_1$, $\textbf{u}_2$ and $\xi$ are randomness generated in the algorithm when $(\bar{x},\bar{y})$ is produced.}
\end{definition}

Note that the minimax problems \eqref{min_max_problem_deterministic} and \eqref{min_max_problem_online} are equivalent to the following minimization problem:
\begin{align}\label{eq:argminrecasting}
\min_{x}\{g(x):=\max _{{y\in\mathcal{Y}}} f({x},{y}) = f(x,y^*(x))\},
\end{align}
where $y^*(x)=\argmax_{y\in\mathcal{Y}}f(x,y)$. Due to our Assumption \ref{assumption:A1-A2-A4}, that $f({x}, \cdot)$ is strongly-concave for any fixed ${x} \in \mathbb{R}^{d_1}$, the maximization problem \(\max_y f(x,y)\) can be solved efficiently and its optimal solution is unique. 
Note that the $\epsilon$-stationary point for \eqref{eq:argminrecasting} is defined as follows. 

\begin{definition}\label{nonconvexstat}
We call $\bar{x}$ an $\epsilon$-stationary point of a differentiable function $g$ if $\bE(\|\nabla g(\bar{x})\|^2_2) \leq \epsilon^2$. 
\end{definition}

In this paper, we focus on analyzing the oracle complexity of algorithms for obtaining an $\epsilon$-stationary point of $g$ as defined in Definition \ref{nonconvexstat}. {This is because optimality in the sense of Definition \ref{nonconvexstat} in turn implies optimality in the sense of Definition \ref{saddleopt}, as we discuss in the following proposition.}


\begin{proposition} \label{equivalence prop}
Under Assumption \ref{assumption:A1-A2-A4}, if a point $\bar{x}$ satisfies $\bE(\|\nabla g(\bar{x})\|^2_2)\leq \epsilon^2$, by using extra $\cO(\kappa d_2 \log(\epsilon^{-1}))$ calls to the zeroth order oracle in the deterministic setting or by using extra  $\cO( d_2 /\epsilon^2)$ calls to the zeroth order oracle in the stochastic setting, a point $(\bar{x},\bar{y})$ can be obtained such that it is an $\epsilon$-stationary solution of the minimax problem as defined in Definition~\ref{saddleopt}. 
\end{proposition}
The proof of this proposition is the same as the proof of Proposition~4.11 in~\cite{lin2019gradient}. We thus omit it for succinctness. 


\section{Zeroth-order Algorithms for Deterministic Minimax Problems}\label{sec:deterministic}

We now present our algorithms for the deterministic minimax problem \eqref{min_max_problem_deterministic}. 

\subsection{Zeroth-Order Gradient Descent Ascent}\label{sec:zo-gda}

Our zeroth-order gradient descent ascent (\texttt{ZO-GDA}) algorithm for solving problem \eqref{min_max_problem_deterministic} is described in Algorithm \ref{ZO-GDA}. The algorithm is similar to the deterministic first-order approach analyzed in~\cite{lin2019gradient} with some crucial differences. Specifically, we require a mini-batch gradient estimator with the choices of the batch size depending on the dimensionality of the problem.
\begin{algorithm}[ht]
\caption{Zeroth-Order Gradient Descent Ascent (\texttt{ZO-GDA})}\label{ZO-GDA}
\begin{algorithmic}
    \STATE \textbf{Initialization}: $(x_{0}, y_{0})$, stepsizes $(\eta_{1}, \eta_{2})$, iteration limit $S > 0$, parameters $\mu_1$ and $\mu_2$. Set $q_1=2(d_1+6)$, $q_2=2(d_2+6)$.
    \FOR{$s=0,\ldots,S-1$}
        \STATE $x_{s+1} \leftarrow x_{s}-\eta_{1} G_{\mu_1}(x_s,y_s,\bu_{1,[q_1]})$ with $\bu_{1,i} \sim N(0,\textbf{1}_{d_1})$, $i\in[q_1]$
        \STATE $y_{s+1} \leftarrow \text{Proj}_{\mathcal{Y}}[y_{s}+\eta_{2}H_{\mu_2}(x_s,y_s,\bu_{2,[q_2]})]$ with $\bu_{2,i} \sim N(0,\textbf{1}_{d_2})$, $i\in[q_2]$
    \ENDFOR
    \STATE Return $(x_1,y_1), \ldots, (x_S,y_S)$.
\end{algorithmic}
\end{algorithm}
The complexity result for \texttt{ZO-GDA} (Algorithm \ref{ZO-GDA}) is provided in Theorem~\ref{ZO_GDA_thm}.
\begin{theorem}\label{ZO_GDA_thm}
Under Assumption \ref{assumption:A1-A2-A4}, by setting
\be\label{eta1}
\eta_1 := \frac{1}{4\times 12^4 \kappa^2 (\kappa+1)^2 (\ell+1)},~~\eta_2:=1/(6\ell),
\ee
and
\be\label{thm3.1-S-mu1-mu2}
   S:=\cO(\kappa^5\epsilon^{-2}),~~\mu_1:= \cO(\epsilon d_1^{-3/2}\kappa^{-2}),~~\mu_2:= \cO(\epsilon d_2^{-3/2}\kappa^{-2}),
\ee
\texttt{ZO-GDA} (Algorithm \ref{ZO-GDA}) returns iterates $(x_1,y_1), \ldots, (x_S,y_S)$ such that there exists an iterate which is an $\epsilon$-stationary point of $g(x) = \max_{y\in\mathcal{Y}} f(x,y)$. That is,~\texttt{ZO-GDA} (Algorithm \ref{ZO-GDA}) returns iterates that satisfy $\min_{s \in \{1,...,S\}} \bE(\|\nabla g(x_s)\|_2^2) \leq \epsilon^2$. Moreover, the total number of calls to the (deterministic) zeroth-order oracle is given by 
$K_{\mathcal{ZO}} = S(q_1+q_2) = \cO (\kappa^5(d_1+d_2)\epsilon^{-2} )$.
\end{theorem}

\begin{remark}\label{remark:tuning}
We see that the total number of calls to the (deterministic) zeroth-order oracle depends linearly on the dimension of the problem. The dependence on $\epsilon$ is the same as that of the corresponding first-order methods~\cite{lin2019gradient}. But, the dependence on the condition number $\kappa$ is increased from $\kappa^2$ to $\kappa^5$ (assuming $d_1$ and $d_2$ are of constant order). This is due to the choice of balancing the various tuning parameters in the zeroth-order setting, in particular $\mu_1$ and $\mu_2$ which are absent in the first-order setting.
\end{remark}

\subsection{Zeroth-Order Gradient Descent Multi-Step Ascent}\label{sec:zo-gdmsa}

We now present our \texttt{ZO-GDMSA} algorithm in Algorithm \ref{ZO-GDMSA}. This algorithm runs $T$ ascent steps, for every descent step. The main idea behind running multiple ascent steps is to better approximate the maximum of the stongly-concave function in each step. Subsequently, picking the number of inner iterations $T$ appropriately helps us obtain improved dependence on $\kappa$ while still maintaining the same dependency on $\epsilon$. We emphasize that~\cite{nouiehed2019solving} used the multi-step ascent approach to handle certain non-convex minimax optimization problems that satisfy the PL condition in the first-order setting.

\begin{algorithm}[ht]
\caption{Zeroth-Order Gradient Descent Multi-Step Ascent (\texttt{ZO-GDMSA})}\label{ZO-GDMSA}
\begin{algorithmic}
 \STATE {\bf Initialization:} $\left(x_0, y_0\right)$, step sizes $\left(\eta_{1}, \eta_{2}\right)$, iteration limit for outer loop $S > 0$, iteration limit for inner loop $T > 0$, parameters $\mu_1$ and $\mu_2$. Set $q_1=2(d_1+6)$ and $q_2=2(d_2+6)$. 
 \FOR{$s=0,\ldots,S-1$}
    \STATE Set $y_0(x_s) \leftarrow y_s$
    \FOR{$t=1,\ldots,T$}
      \STATE $y_t(x_s) \leftarrow \Proj_{\mathcal{Y}}(y_{t-1}(x_s)+\eta_{2} H_{\mu_2}(x_s,y_{t-1}(x_s),\bu_{2,[q_2]}))$ with $\bu_{2,i} \sim N(0,\textbf{1}_{d_2})$, $i\in [q_2]$
    \ENDFOR
    \STATE $y_{s+1} \leftarrow y_{T}(x_s)$
    \STATE $x_{s+1}\leftarrow x_s - \eta_1 G_{\mu_1}(x_s,y_{s+1},\bu_{1,[q_1]})$ with $\bu_{1,i}\sim N(0,\textbf{1}_{d_1})$, $i\in [q_1]$
 \ENDFOR
 \STATE Return $(x_1,y_1), \ldots, (x_S,y_S)$.
\end{algorithmic}
\end{algorithm}

\begin{theorem}\label{ZO_GDmax_thm}
Under Assumption \ref{assumption:A1-A2-A4}, by setting
\be\label{thm3.2-T-eta1-eta2}
{\eta_1=1/(12L_g) = \frac{1}{12(1+\kappa)\ell}}, \ \eta_2=1/(6\ell), \ T= \mathcal{O}(\kappa\log(\epsilon^{-1})),
\ee
and
\be\label{thm3.2-S-mu1-mu2}
\hspace{-0.1in}S = \cO(\kappa\epsilon^{-2}), \ \mu_1 = \cO({\epsilon}d_1^{-3/2}), \ \mu_2=\cO(\kappa^{-1/2}d_2^{-3/2}\epsilon),
\ee
\texttt{ZO-GDMSA} (Algorithm \ref{ZO-GDMSA}) returns iterates $(x_1,y_1), \ldots, (x_S,y_S)$ such that there exists an iterate which is an $\epsilon$-stationary point for $g(x) = \max_{y\in\mathcal{Y}} f(x,y)$. That is,~\texttt{ZO-GDMSA} (Algorithm \ref{ZO-GDMSA}) returns iterates that satisfy $\min_{s \in \{1,...,S\}} \bE(\|\nabla g(x_s)\|_2^2) \leq \epsilon^2$. Moreover, the total number of  calls to the (deterministic) zeroth-order oracle is given by $
K_{\mathcal{ZO}} = Sq_1 + T S q_2 = \cO\Bigl(\kappa\epsilon^{-2}(d_1 + \kappa d_2 \log(\epsilon^{-1}))\Bigr).$
\end{theorem}

\begin{remark}
Compared to Algorithm~\ref{ZO-GDA}, the oracle complexity of Algorithm~\ref{ZO-GDMSA} has improved dependence on $\kappa$ while maintaining the same dependence on $\epsilon$. 
\end{remark}

\section{Zeroth-order Algorithms for Stochastic Minimax Problems}\label{sec:stochasticsetting}

We now consider the stochastic minimax problem \eqref{min_max_problem_online}, under the availability of a stochastic zeroth-order oracle satisfying Assumption~\ref{assumption:stoch} or Assumption  \ref{SGC-assumption}. This scenario is more practical in the context of zeroth-order optimization, as often times, we are able to only observe noisy evaluations of the function \cite{conn2009introduction, audet2017derivative}. Motivated by our analysis of the deterministic case, we now design and analyze the stochastic versions of \texttt{ZO-GDA} and \texttt{ZO-GDMSA}.

We first consider stochastic version of \texttt{ZO-GDA}, which is named \texttt{ZO-SGDA} and presented in Algorithm~\ref{ZO-SGDA}. Under Assumption~\ref{assumption:stoch}, the main difference between Algorithm~\ref{ZO-SGDA} and its deterministic counterpart (Algorithm~\ref{ZO-GDA}) is in the choice of mini-batch size in the zeroth-order gradient estimator. As opposed to the deterministic case, where the mini-batch size is independent of $\epsilon$, in this case, we require a mini-batch size that depends on $\epsilon$. Furthermore, due to the stochastic nature of the problem, the mini-batch size also depends on the noise variance parameter $\sigma^2$. However, under Assumption~\ref{SGC-assumption}, it suffices to have the batch size to be the same as in the deterministic case -- this leads to the rate improvement. The complexity result corresponding to Algorithm~\ref{ZO-SGDA} is provided in Theorem~\ref{ZO_SGDA_thm}.

\begin{algorithm}[ht]
\caption{Zeroth-Order Stochastic Gradient Descent Ascent (\texttt{ZO-SGDA})} \label{ZO-SGDA}
\begin{algorithmic}
 \STATE \textbf{Initialization}: $\left(x_{0}, y_{0}\right)$, step sizes $\left(\eta_{1}, \eta_{2}\right)$, iteration limit $S > 0$, smoothing parameters $\mu_1$ and $\mu_2$. Indices sets $\cM_1$ and $\cM_2$.
 \FOR{$s=0,\ldots,S-1$}
    \STATE \(x_{s+1} \leftarrow x_{s}-\eta_1 \frac{1}{|\cM_1|}\sum_{i\in\cM_1}G_{\mu_1}\left(x_{s}, y_{s},\bu_{1,i},\xi_{i}\right)\) with $\bu_{1,i}\sim N(0,\textbf{1}_{d_1})$
    \STATE \(y_{s+1} \leftarrow \text{Proj}_{\mathcal{Y}}\Bigl[y_{s}+\eta_{2} \frac{1}{|\cM_2|}\sum_{i\in \cM_2}H_{\mu_2}\left(x_{s}, y_{s},\bu_{2,i},\xi_{i}\right)\Bigr]\) with $\bu_{2,i}\sim N(0,\textbf{1}_{d_2})$
 \ENDFOR
 \STATE Return $(x_{1},y_{1}), \ldots, (x_{S},y_{S})$.
\end{algorithmic}
\end{algorithm}

\begin{theorem}\label{ZO_SGDA_thm}
Let $\epsilon \in (0,1)$. Then 
\begin{enumerate}
\item Under Assumptions \ref{assumption:A1-A2-A4} and ~\ref{assumption:stoch}, by setting the parameters $\eta_1,\eta_2$ as in \eqref{eta1}, setting $S,\mu_1,\mu_2$ as in \eqref{thm3.1-S-mu1-mu2}, and setting $|\cM_1| = 4(d_1+6)(\sigma_1^2+1)\epsilon^{-2}$, $|\cM_2| = 4(d_2+6)(\sigma_2^2+1)\epsilon^{-2}$,
\texttt{ZO-SGDA} (Algorithm \ref{ZO-SGDA}) returns iterates $(x_1,y_1), \ldots, (x_S,y_S)$ such that there exists an iterate which is an $\epsilon$-stationary point for $g(x) = \max_{y\in\mathcal{Y}} f(x,y)$. That is,~\texttt{ZO-SGDA} (Algorithm \ref{ZO-SGDA}) returns iterates that satisfy $\min_{s \in \{1,...,S\}} \bE(\|\nabla g(x_s)\|_2^2) \leq \epsilon^2$. Moreover, the total number of calls to the stochastic zeroth-order oracle is given by $K_{\mathcal{SZO}} = S(|\cM_1|+|\cM_2|) = \cO(\kappa^5(d_1\sigma_1^2 +d_2\sigma_2^2)\epsilon^{-4} )$.
\item Under Assumptions \ref{assumption:A1-A2-A4}, ~\ref{assumption:stoch} (only the unbiased part) and ~\ref{SGC-assumption}, by setting $|\cM_1|= \rho_1 (d_1+6)$, $|\cM_2| =\rho_2(d_2+6)$ and setting $
\mu_1 = \cO (\rho_1\ell d_1^{-3/2} )$ and $\mu_2 = \cO (\rho_2\ell d_2^{-3/2} )$,
with other parameters remaining the same, the conclusion in Part 1 holds. In this case, the total number of calls to the stochastic zeroth-order oracle is given by 
$K_{\mathcal{SZO}} = S(|\cM_1|+|\cM_2|) = \cO (\kappa^5(\rho_1 d_1+\rho_2 d_2)\epsilon^{-2} )$.
\end{enumerate}
\end{theorem}

\begin{remark}
Under Assumption~\ref{assumption:stoch}, the $\epsilon$-dependence of Algorithm~\ref{ZO-SGDA} is the same as the first-order counterpart considered in~\cite{lin2019gradient}. However, under Assumption~\ref{SGC-assumption}, the $\epsilon$-dependence is improved and is the same as the deterministic case. 
\end{remark}


The stochastic version of Algorithm~\ref{ZO-GDMSA} is named \texttt{ZO-SGDMSA} and presented in Algorithm~\ref{alg:ZO-SGDMSA}. Its oracle complexity result is provided in Theorem~\ref{ZO_SGDmax_thm}.

\begin{algorithm}[ht]
\caption{Zeroth-Order Stochastic Gradient Multi-Step Descent (\texttt{ZO-SGDMSA})}\label{alg:ZO-SGDMSA}
\begin{algorithmic}
 \STATE \textbf{Initialization}: $\left(x_0, y_0\right)$, step sizes $\left(\eta_{1}, \eta_{2}\right)$, iteration limit for outer loop $S > 0$, iteration limit for inner loop $T > 0$, smoothing parameters $\mu_1$ and $\mu_2$. Indices sets $\cM_1$ and $\cM_2$.
 \FOR{$s=1,\ldots,S-1$}
    \STATE Set $y_0(x_s) \leftarrow y_s$
    \FOR{$t=1,\ldots,T$}
        \STATE \(y_t(x_s) \leftarrow \text{Proj}_{\mathcal{Y}}\Bigl[y_{t-1}(x_s)+\eta_{2}\frac{1}{|\cM_2|}\sum_{i \in \cM_2} H_{\mu_2}(x_s,y_{t-1}(x_s),\bu_{2,i},\xi_i)\Bigr]\) with $\bu_{2,i}\sim N(0,\textbf{1}_{d_2})$
    \ENDFOR
    \STATE $y_{s+1} \leftarrow y_{T}(x_s)$
    \STATE $x_{s+1}\leftarrow x_s - \eta_1 \frac{1}{|\cM_1|} \sum_{i \in \cM_1} G_{\mu_1}(x_s,y_{s+1},\bu_{1,i},\xi_i)$ with $\bu_{1,i}\sim N(0,\textbf{1}_{d_1})$
 \ENDFOR
 \STATE Return $(x_1,y_1),\ldots, (x_S,y_{S})$.
\end{algorithmic}
\end{algorithm}

\begin{theorem}\label{ZO_SGDmax_thm}
Let $\epsilon \in (0,1)$. Then,  
\begin{enumerate} 
\item Under Assumptions \ref{assumption:A1-A2-A4} and ~\ref{assumption:stoch}, by setting $\eta_1,\eta_2$ as in \eqref{thm3.2-T-eta1-eta2}, $S,\mu_1,\mu_2$ as in \eqref{thm3.2-S-mu1-mu2}, and setting $|\cM_1| = 4(d_1+6)(\sigma_1^2+1)\epsilon^{-2}$, $|\cM_2| = 4(d_2+6)(\sigma_2^2+1)\epsilon^{-2}$, 
\texttt{ZO-SGDMSA} (Algorithm~\ref{alg:ZO-SGDMSA}) returns iterates $(x_1,y_1), \ldots, (x_S,y_S)$ such that there exists an iterate which is an $\epsilon$-stationary point for $g(x) = \max_{y\in\mathcal{Y}} f(x,y)$. That is,~\texttt{ZO-SGDMSA} (Algorithm \ref{alg:ZO-SGDMSA}) returns iterates that satisfy $\min_{s \in \{1,...,S\}} \bE(\|\nabla g(x_s)\|_2^2) \leq \epsilon^2$. Moreover, the total number of calls to the stochastic zeroth-order oracle is given by,
$K_{\mathcal{SZO}} = S|\cM_1| + T S |\cM_2| =\cO (\kappa\epsilon^{-4}(d_1\sigma_1^2 + \kappa d_2 \sigma_2^2\log(\epsilon^{-1})) )$.
\item Under Assumptions \ref{assumption:A1-A2-A4}, ~\ref{assumption:stoch} (only the unbiased part) and ~\ref{SGC-assumption}, by setting $|\cM_1|= \rho_1 (d_1+6)$, $|\cM_2| =\rho_2(d_2+6)$ and setting $\mu_1 = \cO (\rho_1\ell d_1^{-3/2})$ and $\mu_2 = \cO (\rho_2\ell d_2^{-3/2})$, with other parameters remaining the same, the conclusion in Part 1 holds. In this case, the total number of calls to the stochastic zeroth-order oracle is given by,
$K_{\mathcal{SZO}} = S|\cM_1| + T S |\cM_2| = \cO (\kappa (\rho_1 d_1 +\kappa \rho_2 d_2 \log(\epsilon^{-1}))\epsilon^{-2} )$. 
\end{enumerate}
\end{theorem}
\begin{remark}
Similar to the deterministic case, we improve the dependence of the oracle complexity on $\kappa$. The dependence on $\epsilon$ and dimensionality remains the same. We emphasize that the use of multiple steps in the ascent part, leads to the improved dependency on $\kappa$ over Algorithm~\ref{ZO-SGDA}.
\end{remark}

\section{Numerical Results}\label{sec:numerical}

We now compare \texttt{ZO-SGDA} and \texttt{ZO-SGDMSA} with their first-order counterparts (i.e., \texttt{SGDA} and \texttt{SGDMSA}) on the distributionally robust optimization problem~\cite{namkoong2016stochastic}. For simplicity, we present the formulation of the problem in the finite-sum setting as:
\[
\min_{x \in \mathbb{R}^d} \max_{y \in \mathcal{Y}} \sum_{i=1}^n y_i \ell_i(x) - r(y),
\]
where $\mathcal{Y} = \{y \in \mathbb{R}^n\mid\sum_{i=1}^n y_i= 1, y_i \geq 0\}$ is the probability simplex; $r(y) = 10 \sum_{i=1}^n (y_i - 1/n)^2$ is a divergence measure regularizing derivations from uniform distribution; $\ell_i(x) = f_1(f_2(x,s_i,z_i))$ where $f_1(x) = \log(1 + x)$, $f_2(x) = \log(1 + \exp[-z_i (x^T s_i)])$, $(s_i,z_i)$ is the feature and label pair of a sample $i$ in the dataset. It is easy to see that the above problem is a nonconvex-strongly concave minimax problem of the from \eqref{min_max_problem_deterministic} with $d_1 = d,d_2 = n$. For the tuning parameters, motivated by our theoretical results, we set the batch size $|\mathcal{M}_1| = d_1/\epsilon^2$ and $|\mathcal{M}_2| = d_2/\epsilon^2$ with $\epsilon = 0.01$. {For \texttt{ZO-SGDA}, we choose $\eta_1 = \eta_2 = 0.01$, and for \texttt{ZO-SGDMSA}, we choose $\eta_1 = 0.001$ and $\eta_2 = 0.01$. For \texttt{ZO-GDA} and \texttt{ZO-SGDA}, according to Theroems 1 and 3, we chose 
\[
\mu_1 = 1.5\epsilon d_1^{-3/2}\kappa^{-2} \quad \mu_2 = 1.5\epsilon d_2^{-3/2}\kappa^{-2}.
\]
For \texttt{ZO-GDMSA} and \texttt{ZO-SGDMSA}, according to Theorems 2 and 4, we chose
\[
\mu_1 = 1.5\epsilon d_1^{-3/2} \quad \mu_2 = 1.5\kappa^{-1 / 2} d_{2}^{-3 / 2} \epsilon.
\]
For our tested problems, we have $\kappa^3 = 10$ (see, for example \cite{xu2020enhanced}). The $\mu_1$ and $\mu_2$ in the above equations are usually at the order of $10^{-5}$ to $10^{-6}$.} For \texttt{SGDA} and \texttt{SGDMSA}, we choose the same stepsize as \texttt{ZO-SGDA} and \texttt{ZO-SGDMSA} and set $|\mathcal{M}_1| = 1/\epsilon^2$ and $|\mathcal{M}_2| = 1/\epsilon^2$. We stop the iteration when $\|\nabla g(x_s)\|_2 \leq \epsilon$, based on our theoretical analysis. We test our algorithms on the following datasets from UCI ML-repository~\cite{Dua:2019} and LIBSVM~\cite{CC01a}: A9A \footnote{\url{https://www.csie.ntu.edu.tw/~cjlin/libsvmtools/datasets/binary/a9a}}, Mushroom \footnote{\url{https://archive.ics.uci.edu/ml/datasets/mushroom}},W8A \footnote{\url{https://www.csie.ntu.edu.tw/~cjlin/libsvmtools/datasets/binary/w8a}} and Colon-cancer gene expression dataset\footnote{\url{https://www.csie.ntu.edu.tw/~cjlin/libsvmtools/datasets/binary/colon-cancer.bz2}}. In order to perform distributionally robust optimization, we sample the dataset such that the positive and negative label ratio is 1:3. Details of these datasets are provided in Table \ref{tab:datasets}. All the experiments were run on Google Colab Python 3.5 Notebook. We also remark that we cannot compare empirically to~\cite{liu2019min} as they consider constrained minimax problems. In Figure~\ref{fig:mainfig}, we plot the value of the objective versus iteration number and the value of gradient size versus iteration number. We find that the proposed zeroth-order methods perform favorably to their respective first-order counterparts in terms of both the objective value and the norm of the gradient of the function $g$, as measured by iteration count. {It should be noted that to obtain this comparable behavior, the zeroth-order method uses a mini-batch of samples that is proportional to the dimension (recall our choice of $|\mathcal{M}_1|$ and $|\mathcal{M}_2|$) above) in each iteration, which results in the number of calls to the zeroth-order oracle of the order as illustrated in our theoretical results.}

\begin{table}[h!]
\centering{
\begin{tabular}{|l|l|l|l|}
\hline
Dataset & Samples & Features & T/F ratio \\
\hline
   A9A  &   200      &    123      &      
     1:3   \\
        \hline
    Mushroom   &    100     &    22      &   1:3        \\
    \hline
  W8A      &    100     &     300      &  1:3\\
  \hline
  Colon-cancer      &    200     &     500      &  1:3\\
  \hline        
\end{tabular}
\caption{Details of the datasets \cite{Dua:2019}.}\label{tab:datasets}
}
\end{table}

\begin{figure}[htbp]
\centering   
\includegraphics[scale=0.18]{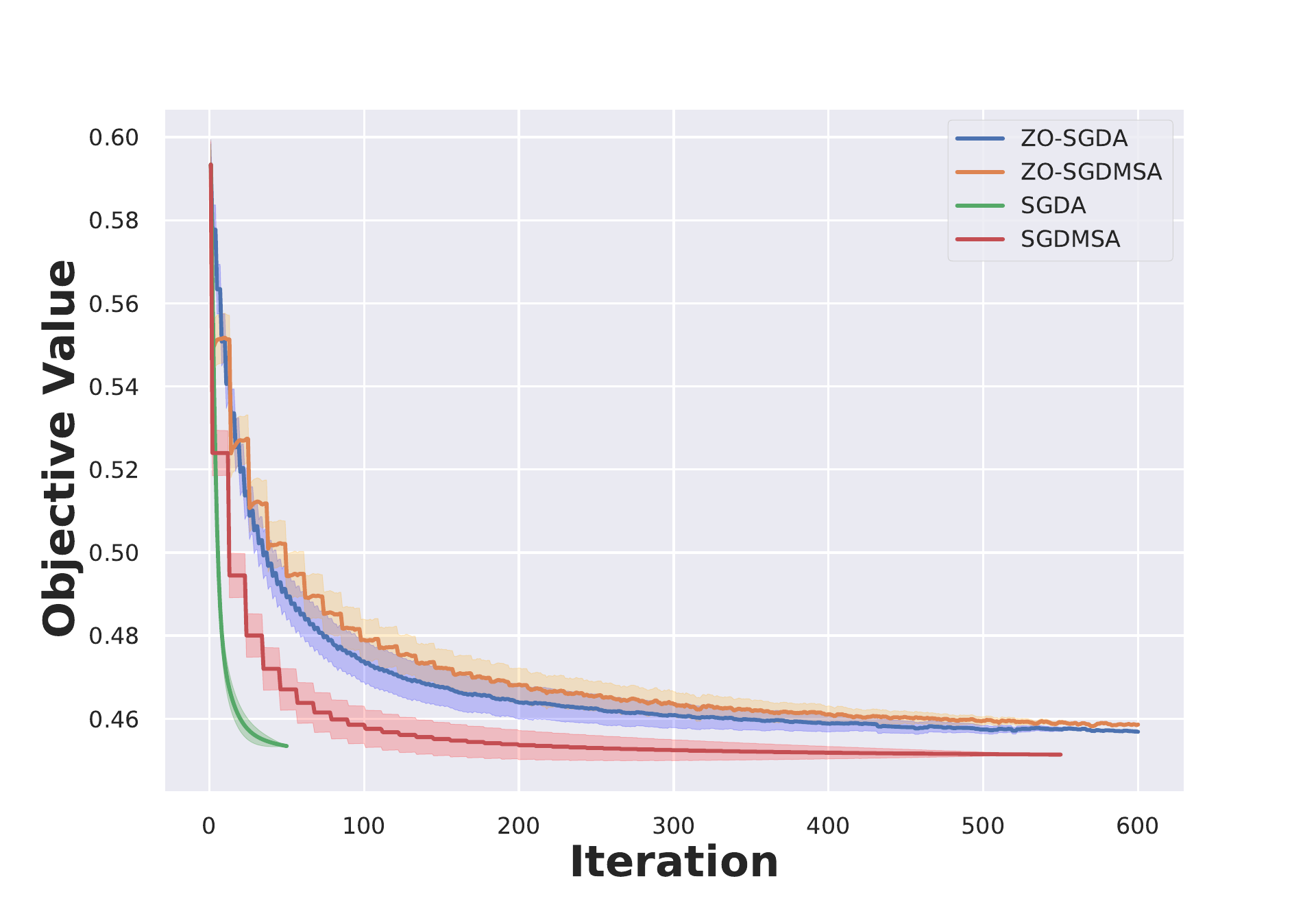}
\includegraphics[scale=0.18]{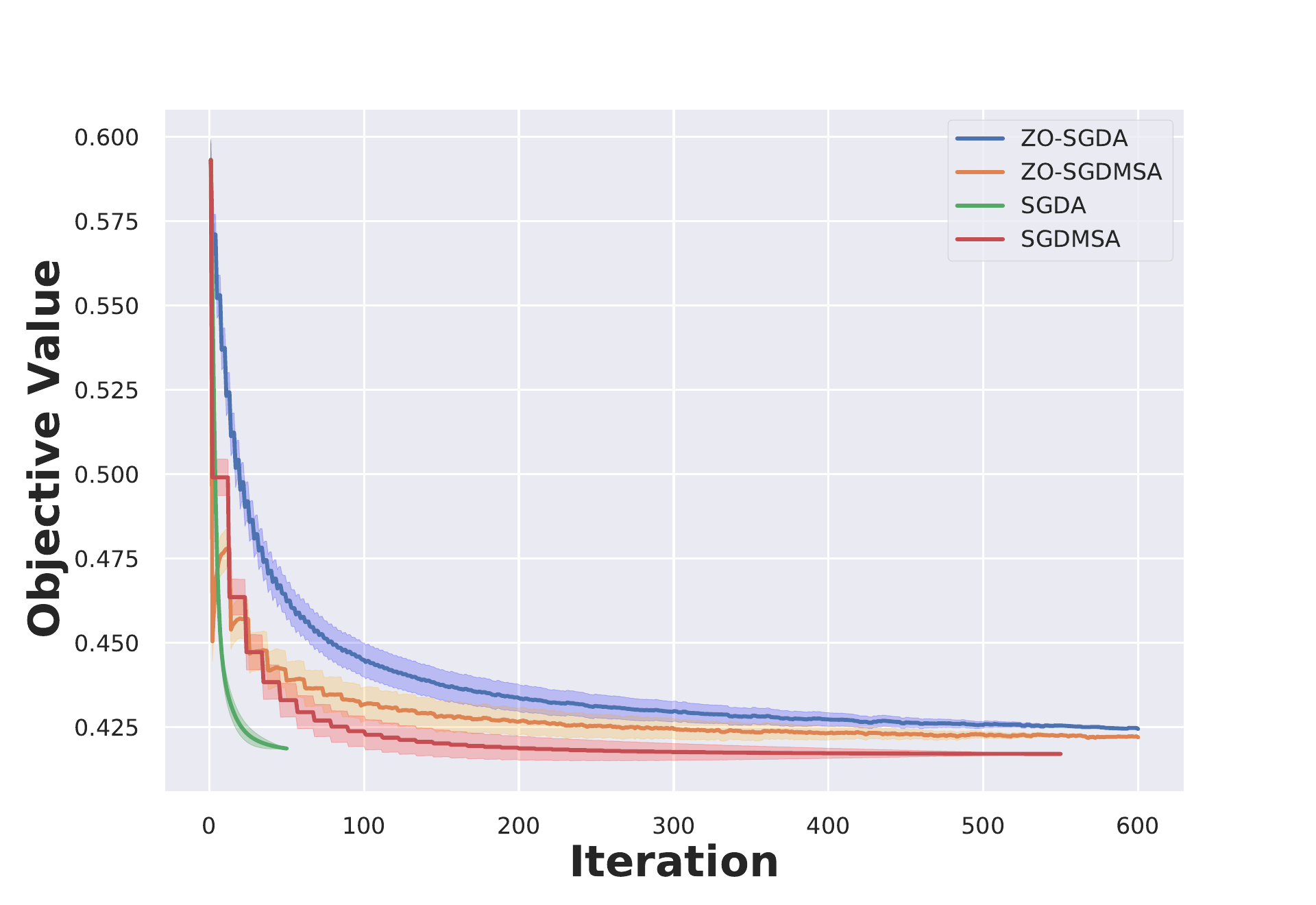}
\includegraphics[scale=0.18]{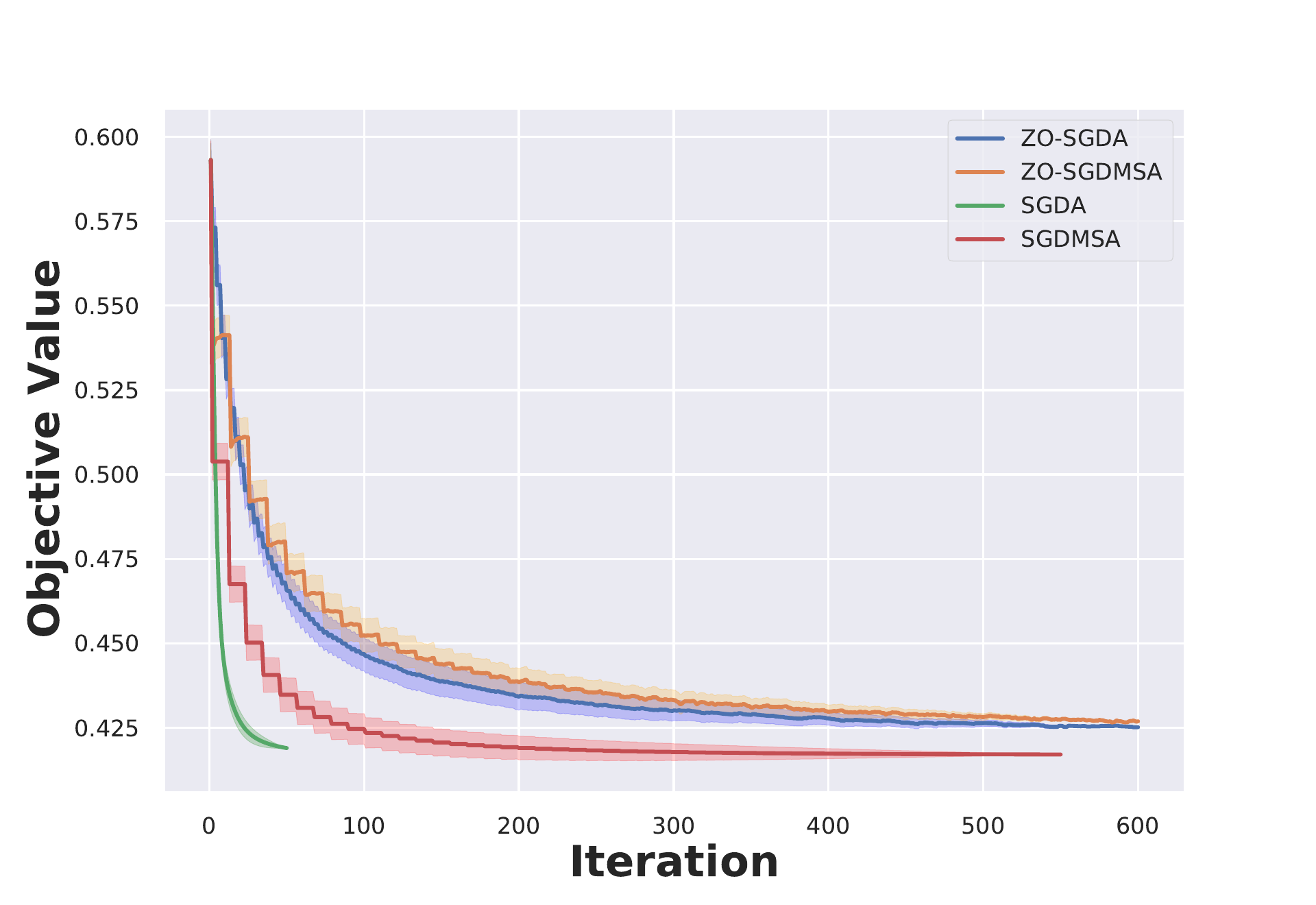}
\includegraphics[scale=0.18]{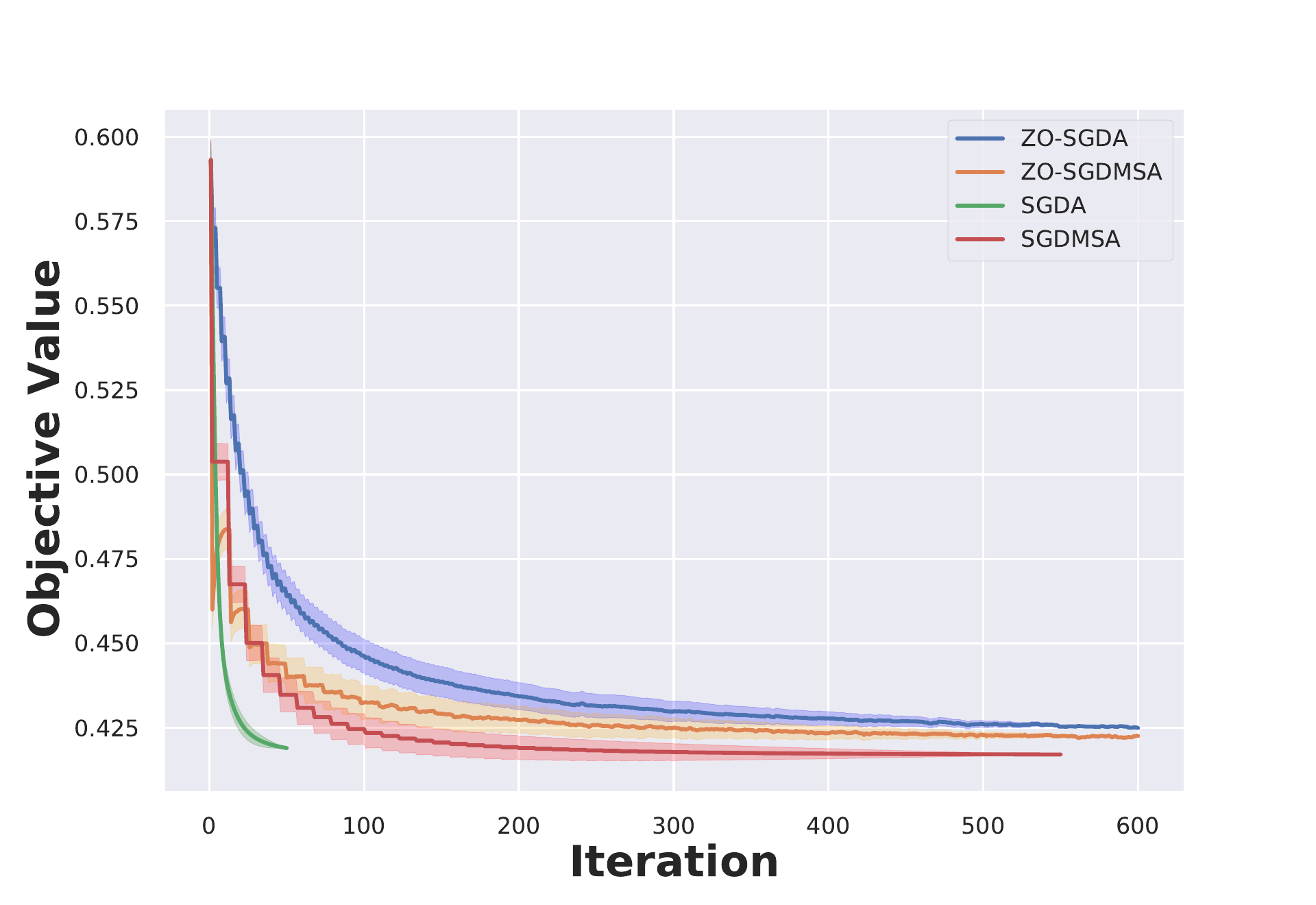}
\includegraphics[scale=0.18]{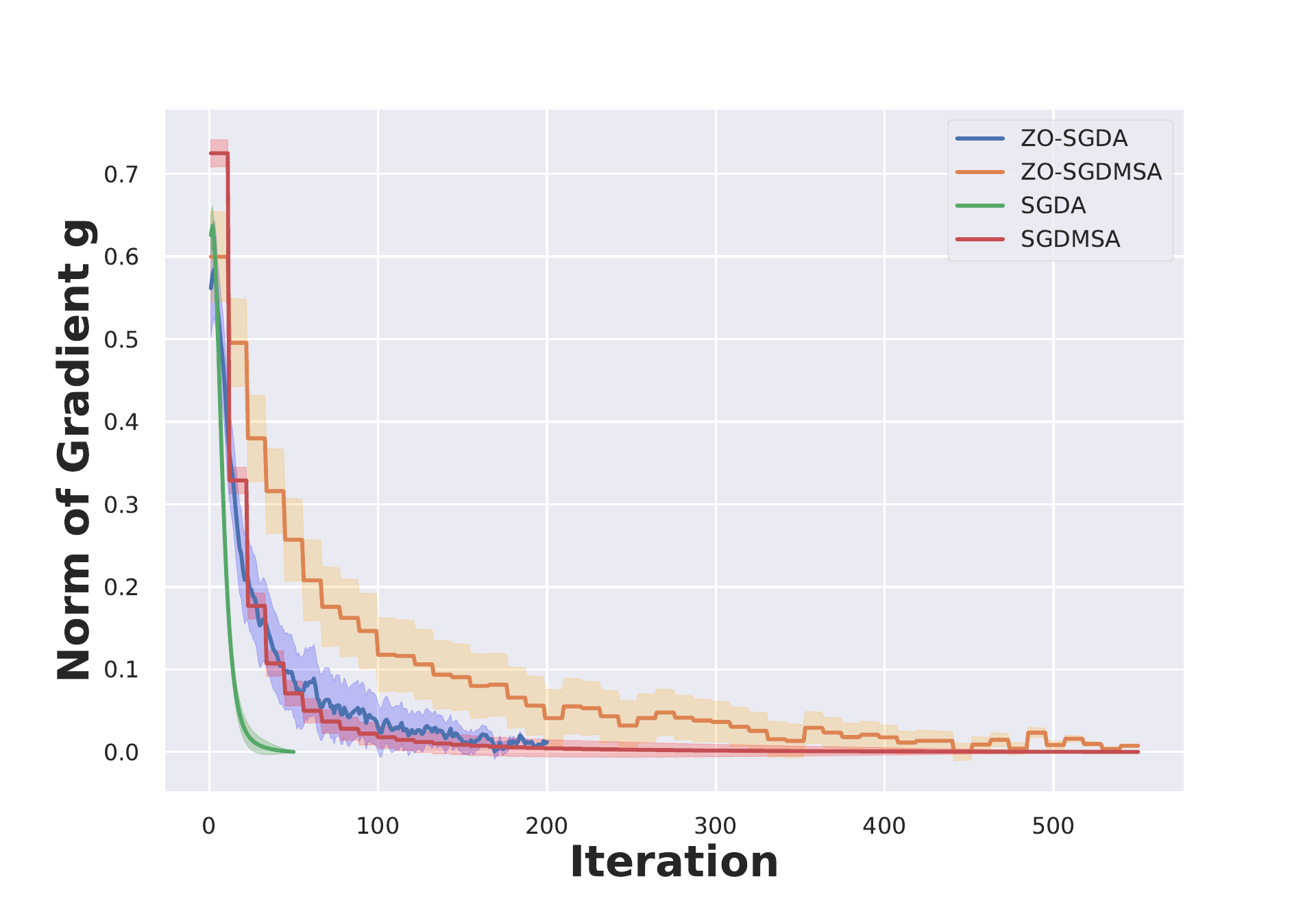}
\includegraphics[scale=0.18]{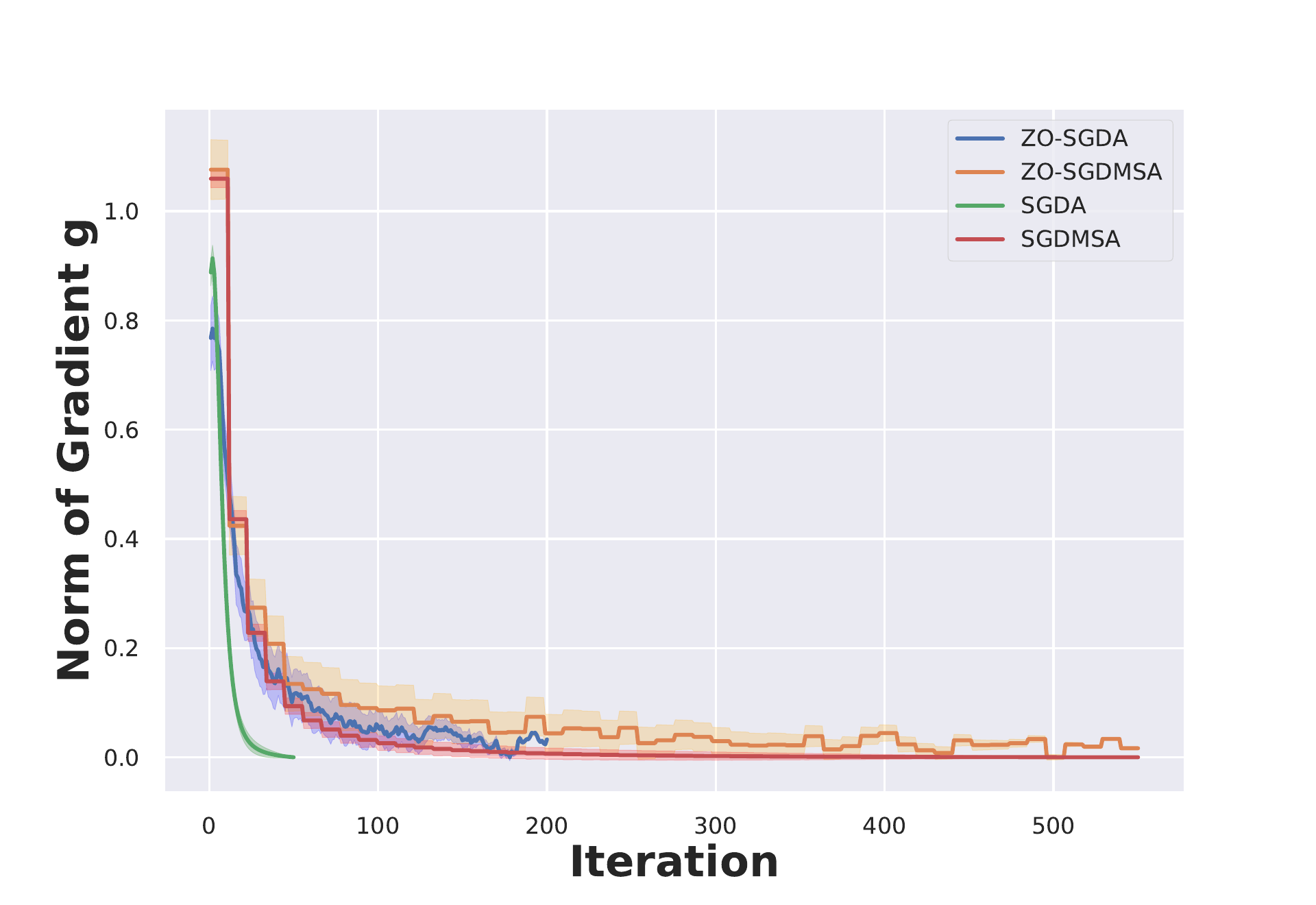}
\includegraphics[scale=0.18]{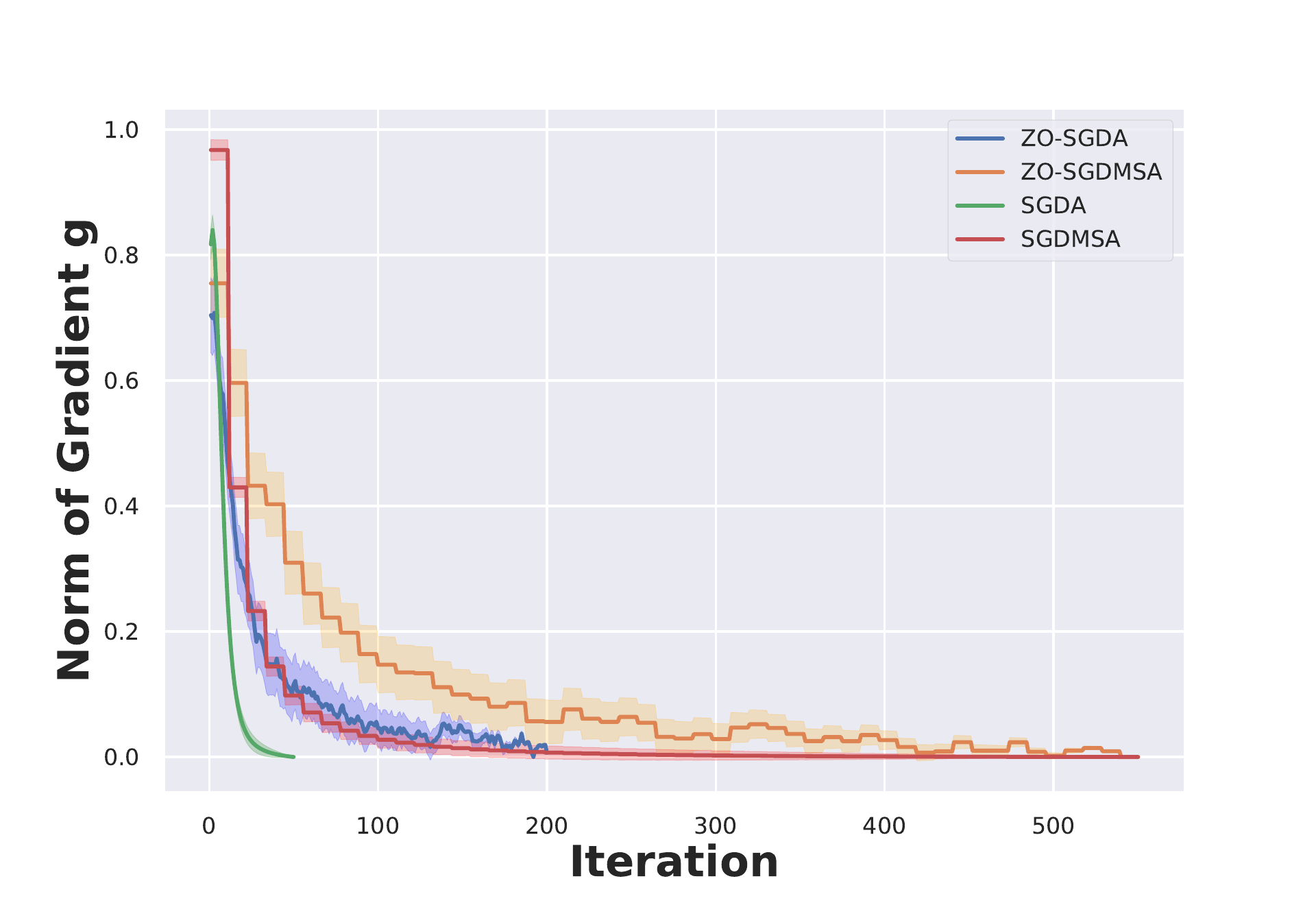}
\includegraphics[scale=0.18]{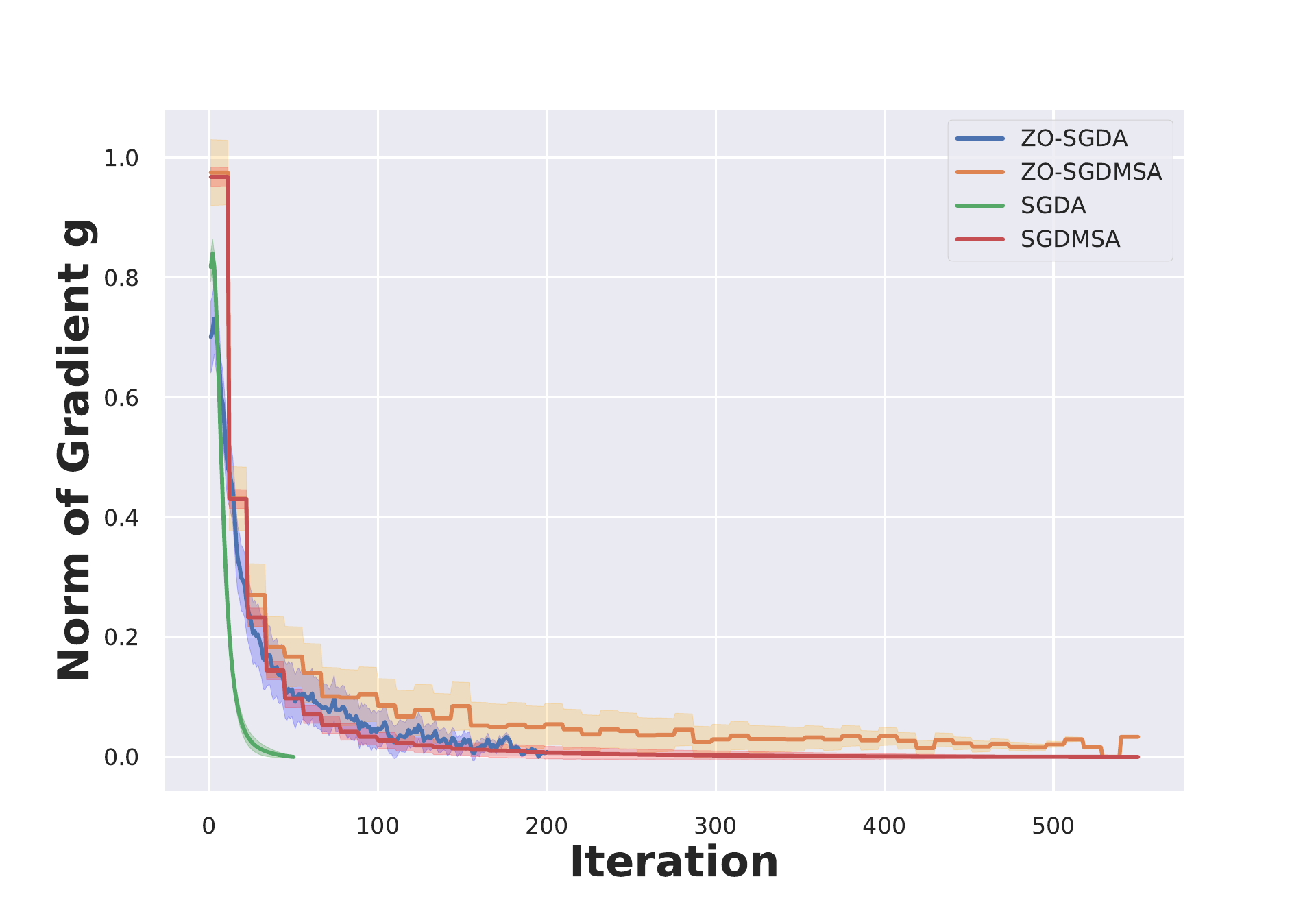}
\caption{Performance of \texttt{ZO-SGDA} and \texttt{ZO-SGDMSA} in comparison to their first-order counterparts. The results in the four rows respectively correpond to the following datasets: A9A dataset, Mushroom dataset, W8A dataset and Colon Cancer dataset. The results correspond to average over 500 trails. }
\label{fig:mainfig}
\end{figure}

\section{Conclusions}\label{discussion} 

In this paper, we designed and analyzed zeroth-order algorithms for deterministic and stochastic nonconvex minimax problems. Specifically, we considered two types of algorithms: zeroth-order gradient descent ascent algorithm and a modified version of it with multiple ascent steps following each descent step. We obtained oracle complexities for both algorithms that match the performance of comparable first-order algorithms, up to unavoidable dimensionality factors. Our orcale complexities are better than that of existing methods under the same assumptions. Future works include to explore lower bounds for zeroth-order nonconvex minimax optimization problems, and to explore structural constraints to obtain improved dimensionality dependence in our results. 

\section*{Acknowledgements}
A preliminary version (4 pages) of this paper appeared in the 2021 ICML workshop ``Beyond first-order methods in ML systems'' (\url{https://sites.google.com/view/optml-icml2021}). There is no proceeding for this workshop. K. Balasubramanian was supported in part by NSF Grant DMS-2053918 and UC Davis CeDAR (Center for Data Science and Artificial Intelligence Research) Innovative Data Science Seed Funding Program. S. Ma was supported in part by NSF grants DMS-1953210 and CCF-2007797, and UC Davis CeDAR Innovative Data Science Seed Funding Program. M. Razaviyayn was supported in part by the NSF CAREER Award CCF-2144985 and the AFOSR Young Investigator Program Award.

\bibliographystyle{alpha}
\bibliography{reference}   
\appendix

\section{Technical Preparations}

In this section we present some technical results that will be used in our subsequent convergence analysis. First, we need the follow elementary results regarding random variables.

\begin{lemma}\label{lemma:random-var}
\begin{itemize}
\item For i.i.d. random (vector) variables $\bX_i$, $i=1,\ldots,N$ with zero mean, we have $\bE(\|\frac{1}{N}\sum_{i=1}^N\bX_i\|_2^2) = \frac{1}{N}\bE (\|\bX_1\|_2^2)$.
\item For random (vector) variable $\bX$, we have $\bE (\|\bX - \bE \bX\|_2^2) = \bE(\|\bX\|_2^2) - (\|\bE\bX\|_2^2) \leq \bE(\|\bX\|_2^2)$ and $\|\bE\bX\|_2^2 \leq \bE (\|\bX\|_2^2)$.
\end{itemize}
\end{lemma}

The following results regarding Lipschitz and strongly convex functions are also useful.
\begin{lemma} (Lemma 1.2.3, Theorem 2.1.8, Theorem 2.1.10 in \cite{NesterovConvexBook2004})\label{lemma:smooth-stronglyconvex}
\begin{itemize}
	\item Suppose a function $h$ is $L_h$ gradient-Lipschitz and has a unique maximizer $x^*$. Then, for any $x$, we have:
	\be\label{Lip-gradient} \frac{1}{2L_h}\|\nabla h(x)\|_2^2 \leq h(x^*) - h(x) \leq \frac{L_h}{2}\|x - x^*\|_2^2. \ee
    \item Suppose a function $h$ is $\tau_h$ strongly concave and has a unique maximizer $x^*$. Then, for any $x$, we have:
    \be\label{strongly-convex} \frac{\tau_h}{2}\|x - x^*\|_2^2\leq h(x^*) - h(x) \leq \frac{1}{2\tau_h}\|\nabla h(x)\|_2^2.\ee
\end{itemize}
\end{lemma}

The following lemmas are from existing literature and we omit their proofs.
\begin{lemma}(Lemma 4.3 in \cite{lin2019gradient})\label{lemma:smoothed convexity}
The function $g(\cdot) := \max_{y\in\cY} f(\cdot,y)$ is $L_g:= (\ell+\kappa\ell)$-smooth with $\nabla g(x)=\nabla_x f(x,y^*(x))$. Moreover, $y^*(x)=\argmax_{y\in\cY}f(\cdot,y)$ is $\kappa$-Lipschitz.
\end{lemma}

\begin{lemma}\cite{nesterov2017random}\label{lemma:f-mu-convex}
$f_{\mu}(x) = \bE_{\bu} f_{\mu}(x+\mu \bu)$ is a convex function, if $f(x)$ is convex.
\end{lemma}

\begin{lemma}(Theorem 1 in \cite{nesterov2017random})\label{lemma:NS17-Eq19}
Under Assumption \ref{assumption:A1-A2-A4}, it holds that
\[
|f_{\mu_2}(x,y)-f(x,y)| \leq \frac{\mu_2^2}{2}\ell d_2, \forall x\in\br^{d_1},y\in\cY.
\]
\end{lemma}

\begin{lemma}(Lemma 3 in \cite{nesterov2017random})\label{lemma:diff-f-fmu}
Under Assumption \ref{assumption:A1-A2-A4}, it holds that
\[\|\nabla_x f_{\mu_1}(x,y) - \nabla_x f(x,y)\|_2^2 \leq \frac{\mu_1^2}{4}\ell^2(d_1+3)^3, \ \|\nabla_y f_{\mu_2}(x,y) - \nabla_y f(x,y)\|_2^2 \leq \frac{\mu_2^2}{4}\ell^2(d_2+3)^3.
\]
\end{lemma}

\begin{lemma}(Lemma 4 in \cite{nesterov2017random})\label{lemma:diff-f-fmu-next}
Under Assumption \ref{assumption:A1-A2-A4}, it holds that
\[\|\nabla_x f(x,y)\|_2^2 \leq 2\|\nabla_x f_{\mu_1}(x,y)\|_2^2 + \ell^2\mu_1^2(d_1+3)^3/2.
\]
\end{lemma}

\begin{lemma}(Theorem 4 in \cite{nesterov2017random})\label{lemma:determine g bound}
Under Assumptions \ref{assumption:A1-A2-A4} and \ref{assumption:stoch}, we have
\[
\bE_{\bu_1}\|G_{\mu_1}(x,y,\bu_1)\|_2^2 \leq 2(d_1+4)\|\nabla_x f(x,y)\|_2^2 + \mu_1^2\ell^2(d_1+6)^3/2,
\]
\[
\bE_{\bu_2}\|H_{\mu_2}(x,y,\bu_2)\|_2^2 \leq 2(d_2+4)\|\nabla_y f(x,y)\|_2^2 + \mu_2^2\ell^2(d_2+6)^3/2.
\]
\end{lemma}

\begin{lemma}\cite{balasubramanian2018zeroth}\label{lemma:stochastic g bound}
Under Assumptions \ref{assumption:A1-A2-A4} and \ref{assumption:stoch}, we have
\[
\bE_{\bu_1,\xi}\|G_{\mu_1}(x,y,\bu_1,\xi)\|_2^2 \leq \frac{\mu_1^2\ell^2}{2}(d_1+6)^3 + 2\Bigl[\|\nabla_x f(x,y)\|_2^2 + \sigma_1^2\Bigr](d_1+4),
\]
\[
\bE_{\bu_2,\xi}\|H_{\mu_2}(x,y,\bu_2,\xi)\|_2^2 \leq \frac{\mu_2^2\ell^2}{2}(d_2+6)^3 + 2\Bigl[\|\nabla_y f(x,y)\|_2^2 + \sigma_2^2\Bigr](d_2+4).
\]
\end{lemma}

We now bound the size of the mini-batch zeroth-order gradient estimator \eqref{def-G-mu-H-mu-q}.
\begin{lemma}\label{lemma:Averaged Upper bound}
Under Assumption \ref{assumption:A1-A2-A4} and choosing $q_1 = 2(d_1+6)$, $q_2 = 2(d_2+6)$. For any $x\in\br^{d_1}, y\in\cY$, we have
\be\label{lemma:Averaged Upper bound-inequality}
\ba{ll}
\bE_{\bu_{1,[q_1]}}\| G_{\mu_1}(x,y,\bu_{1,[q_1]})\|_2^2 & \leq 3\|\nabla_x f(x,y)\|_2^2 + \mu_1^2\ell^2 (d_1+6)^3,\\
\bE_{\bu_{2,[q_2]}}\| H_{\mu_2}(x,y,\bu_{2,[q_2]})\|_2^2 & \leq 3\|\nabla_y f(x,y)\|_2^2 + \mu_2^2\ell^2 (d_2+6)^3.
\ea
\ee
\end{lemma}

\begin{proof}
Since $\bE_{\bu_{1,[q_1]}} G_{\mu_1}(x,y,\bu_{1,[q_1]})= \nabla_x f_{\mu_1}(x,y)$, we have
\[\ba{ll}
     & \bE_{\bu_{1,[q_1]}}\|G_{\mu_1}(x,y,\bu_{1,[q_1]})\|_2^2 \\
=    & \bE_{\bu_{1,[q_1]}}\|G_{\mu_1}(x,y,\bu_{1,[q_1]}) - \nabla_x f_{\mu_1}(x,y)\|_2^2 + \|\nabla_x f_{\mu_1}(x,y)\|_2^2 \\
=    & \frac{1}{q_1}\bE_{\bu_1}\|G_{\mu_1}(x,y,\bu_1) - \nabla_x f_{\mu_1}(x,y)\|_2^2 + \|\nabla_x f_{\mu_1}(x,y)\|_2^2 \\
\leq & \frac{1}{q_1}\bE_{\bu_1}\|G_{\mu_1}(x,y,\bu_1)\|_2^2 + 2\|\nabla_x f(x,y)\|_2^2 + 2\|\nabla_x f_{\mu_1}(x,y) - \nabla_x f(x,y)\|_2^2\\
\leq & \frac{2(d_1+4)}{q_1}\|\nabla_x f(x,y)\|_2^2 + 2\|\nabla_x f(x,y)\|_2^2 + \frac{\ell^2\mu_1^2(d_1+3)^3}{2} + \frac{\mu_1^2\ell^2 (d_1+6)^3}{2q_1},
\ea\]
where the second equality is due to Lemma \ref{lemma:random-var}, and the last inequality is due to Lemma \ref{lemma:determine g bound}. Thus, the first inequality in \eqref{lemma:Averaged Upper bound-inequality} is obtained by noting $q_1 = 2(d_1+6)$. The other inequality can be proved similarly and we omit the details for succinctness.
\end{proof}

A similar result can be obtained for the stochastic zeroth-order gradient estimator \eqref{def-G-mu-H-mu-sto-M}.
\begin{lemma}\label{lemma:mini batch upper bound}
Under Assumptions \ref{assumption:A1-A2-A4} and \ref{assumption:stoch}, for given tolerance $\epsilon\in(0,1)$, by choosing $|\cM_1| = 4(d_1+6)(\sigma_1^2+1)\epsilon^{-2}$, $|\cM_2| = 4(d_2+6)(\sigma_2^2+1)\epsilon^{-2}$, for any $x\in\br^{d_1}$, $y\in\cY$, we have:
\be\label{lemma:mini batch upper bound-inequality}
\ba{ll}
\bE_{\bu_{\cM_1},\xi_{\cM_1}}\|G_{\mu_1}(x,y,\bu_{\cM_1},\xi_{\cM_1})\|_2^2 &\leq 3\|\nabla_x f(x,y)\|_2^2 +\varrho_1(\epsilon,\mu_1),\\
\bE_{\bu_{\cM_2},\xi_{\cM_2}}\|H_{\mu_2}(x,y,\bu_{\cM_2},\xi_{\cM_2})\|_2^2 &\leq 3\|\nabla_y f(x,y)\|_2^2 +\varrho_2(\epsilon,\mu_2).
\ea\ee
where $\varrho_1(\epsilon,\mu_1):=\epsilon^2/2 + \mu_1^2 \ell^2 (d_1+3)^3/2 + \mu_1^2\ell^2(d_1+6)^2\epsilon^2/8$, and $\varrho_2(\epsilon,\mu_2):=\epsilon^2/2 + \mu_2^2 \ell^2 (d_2+3)^3/2 + \mu_2^2\ell^2(d_2+6)^2\epsilon^2/8$.
\end{lemma}

\begin{proof}
Since $\bE_{\xi_{\cM_1},\bu_{\cM_1}}G_{\mu_1}(x,y,\bu_{\cM_1},\xi_{\cM_1}) = \nabla_x f_{\mu_1}(x,y)$, we have
\[\ba{ll}
     & \bE_{\xi_{\cM_1},\bu_{\cM_1}}\|G_{\mu_1}(x,y,\bu_{\cM_1},\xi_{\cM_1})\|_2^2\\
= &~\bE_{\xi_{\cM_1},\bu_{\cM_1}}\|G_{\mu_1}(x,y,\bu_{\cM_1},\xi_{\cM_1}) - \nabla_x f_{\mu_1}(x,y)\|_2^2 + \|\nabla_x f_{\mu_1}(x,y)\|_2^2\\
= &~\frac{1}{|\cM_1|}\bE_{\xi_1,\bu_1}\|G_{\mu_1}(x,y,\bu_1,\xi)\|_2^2  +\|\nabla_x f_{\mu_1}(x,y)\|_2^2\\
\leq &~\frac{1}{|\cM_1|}\Bigl[\frac{\mu_1^2L_1^2}{2}(d_1+6)^3 + 2\Bigl[\|\nabla_x f(x,y)\|_2^2 + \sigma_1^2\Bigr](d_1+4)\Bigr] + 2\|\nabla_x f(x,y)\|_2^2 + \mu_1^2\ell^2(d_1+3)^3/2\\
\leq &~\frac{2(d_1+4)}{|\cM_1|}\|\nabla_x f(x,y)\|_2^2 + 2\|\nabla_x f(x,y)\|_2^2 + \frac{2(d_1+4)\sigma_1^2}{|\cM_1|} + \mu_1^2\ell^2(d_1+3)^3/2 + \frac{\mu_1^2\ell^2}{2|\cM_1|}(d_1+6)^3,
\ea\]
where the second equality is due to Lemma \ref{lemma:random-var}, the first inequality is due to Lemma \ref{lemma:stochastic g bound} and \ref{lemma:diff-f-fmu}. Substituting $|\cM_1| = 4(d_1+6)(\sigma_1^2+1)\epsilon^{-2}$
proves the first inequality in \eqref{lemma:mini batch upper bound-inequality}. The other inequality can be proved similarly and we omit the details for succinctness.
\end{proof}


The following result shows that $\nabla_x f_{\mu_1}(x,y)$ is Lipschitz continuous with respect to $y$.
\begin{lemma}\label{lemma:smooth of smoothed function}
Under Assumption \ref{assumption:A1-A2-A4}, for any $x\in\br^{d_1}$, $y_1,y_2\in\cY$, it holds that
\[
\|\nabla_x f_{\mu_1}(x,y_1) - \nabla_x f_{\mu_1}(x,y_2)\|_2 \leq \ell\|y_1-y_2\|_2.
\]
\end{lemma}

\begin{proof}
Following the definition of $f_{\mu_1}$, Assumption \ref{assumption:A1-A2-A4}, and Jensen's inequality, it holds that
\[\ba{ll}
   &~\|\nabla_x f_{\mu_1}(\bx,y_1) - \nabla_x f_{\mu_1}(\bx,y_2)\|_2 \\
=  &~\|\bE_{\bu_1}\nabla_x f(\bx+\mu_1 \bu_1 ,y_1) - \bE_{\bu_1} \nabla_x f(\bx+\mu_1 \bu,y_2)\|_2 \\
\leq &~\bE_{\bu_1}\|\nabla_x f(\bx+\mu_1 \bu_1 ,y_1) -  \nabla_x f(\bx+\mu_1 \bu_1,y_2)\|_2 \\
{\leq} & \ell\|y_1-y_2\|_2,
\ea\]
which proves the desired result.
\end{proof}

We now present the corresponding results with the Strong Growth Condition in Assumption~\ref{SGC-assumption}.
\begin{lemma}
The variance of the mini-batch stochastic gradient under strong growth condition is bounded by
\be\label{mini-batch-sgc}
\ba{l}
	\bE \Bigl\|\frac{1}{\mid\cM_1\mid}\sum_{i=1}^{|\cM_1|}\nabla_x F(x,y,\xi_i)\Bigr\|_2^2 \leq \frac{\rho_1}{\mid\cM_1\mid}\|\nabla_x f(x,y)\|_2^2 \leq \rho_1 \|\nabla_x f(x,y)\|_2^2, \\
	\bE \Bigl\|\frac{1}{\mid\cM_2\mid}\sum_{i=1}^{\mid\cM_2\mid}\nabla_y F(x,y,\xi_i)\Bigr\|_2^2 \leq \frac{\rho_2}{\mid\cM_2\mid}\|\nabla_y f(x,y)\|_2^2 \leq \rho_2 \|\nabla_y f(x,y)\|_2^2.
\ea\ee
\end{lemma}

\begin{proof}
Using Cauchy Schwartz inequality we have:
\[
\ba{lll}
\bE \Bigl\|\frac{1}{\mid\cM_1\mid}\sum_{i=1}^{|\cM_1|}\nabla_x f(x,y,\xi_i)\Bigr\|_2^2 &= \frac{1}{\mid \cM_1\mid^2} |\cM_1|\bE \|\nabla_x f(x,y,\xi)\|_2^2 \\
&= \frac{1}{|\cM_1|} \bE \|\nabla_x f(x,y,\xi)\|_2^2\\
&\leq  \frac{\rho_1}{|\cM_1|} \|\nabla_x f(x,y)\|_2^2 \leq \rho_1\|\nabla_x f(x,y)\|_2^2,
\ea
\]
where the last inequality is based on $|\cM_1| \geq 1$. The other inequality can be proved similarly. 
\end{proof}

\begin{lemma}[ZO Gradient Under Strong Growth Condition]\label{lemma:SGC-zo-gradient-bound}
	Under SGC, by choosing $|\cM_1| = \rho_1(d_1+6)$ and $|\cM_2| = \rho_2(d_2+6)$, we have
	\be \label{ZO_GD_SGC_bound}
	\ba{l}
	\bE_{\xi_{\cM_1},\bu_{\cM_1}}\|G_{\mu_1}(x,y,\bu_{\cM_1},\xi_{\cM_1})\|_2^2 
\leq 3\|\nabla_x f(x,y)\|_2^2 + \varrho_1(\mu_1,\rho_1) \\
\bE_{\xi_{\cM_2},\bu_{\cM_2}}\|H_{\mu_1}(x,y,\bu_{\cM_2},\xi_{\cM_2})\|_2^2 
\leq 3\|\nabla_y f(x,y)\|_2^2 + \varrho_2(\mu_2,\rho_2),
\ea\ee
with $\varrho_1(\mu_1,\rho_1) = \frac{\mu_1^2 }{\rho_1}\ell^2 (d_1+6)+ \mu_1^2 \ell^2 (d_1+3)^3/2
$ and $\varrho_2(\mu_2,\rho_2) = \frac{\mu_2^2 }{\rho_2}\ell^2 (d_2+6)+ \mu_2^2 \ell^2 (d_2+3)^3/2$.
\end{lemma}

\begin{proof}
Since $\bE_{\xi_{\cM_1},\bu_{\cM_1}}G_{\mu_1}(x,y,\bu_{\cM_1},\xi_{\cM_1}) = \nabla_x f_{\mu_1}(x,y)$, we have
	\[\ba{ll}
     & \bE_{\xi_{\cM_1},\bu_{\cM_1}}\|G_{\mu_1}(x,y,\bu_{\cM_1},\xi_{\cM_1})\|_2^2\\
= &~\bE_{\xi_{\cM_1},\bu_{\cM_1}}\|G_{\mu_1}(x,y,\bu_{\cM_1},\xi_{\cM_1}) - \nabla_x f_{\mu_1}(x,y)\|_2^2 + \|\nabla_x f_{\mu_1}(x,y)\|_2^2\\
\leq &~\frac{1}{|\cM_1|}\bE_{\xi_1,\bu_1}\|G_{\mu_1}(x,y,\bu_1,\xi)\|_2^2  +\|\nabla_x f_{\mu_1}(x,y)\|_2^2\\
\leq &~\frac{d_1+4}{|\cM_1|}\bE_{\xi}\|\nabla_x f(x,y,\xi)\|_2^2 + \frac{\mu_1^2 L^2 (d_1+6)^2}{2|\cM_1|} + 2\|\nabla_x f(x,y)\|_2^2 + \mu_1^2\ell^2(d_1+3)^3/2\\
\leq &~\frac{\rho_1(d_1+4)}{|\cM_1|} \|\nabla_x f(x,y)\|_2^2 + 2\|\nabla_x f(x,y)\|_2^2 + \frac{\mu_1^2 \ell^2 (d_1+6)^2}{2|\cM_1|} + \mu_1^2\ell^2(d_1+3)^3/2,
\ea\]
where the last inequality follows from Assumption \ref{SGC-assumption}. By using $|\cM_1| = \rho_1(d_1+6)$, we proved the first inequality of \eqref{ZO_GD_SGC_bound}. The other inequality can be proved similarly. 
\end{proof}

%
\section{Convergence analysis of \texttt{ZO-GDA} (Algorithm \ref{ZO-GDA})}

We first show the following lemma.

\begin{lemma}\label{lemma:ZO-GDA-useful-inequality}
Assume $\{(x_s,y_s)\}$ is the sequence generated by Algorithm \ref{ZO-GDA}. By setting $\eta_2 = 1/(6\ell)$, the following inequality holds:
\be\label{lemma:ZO-GDA-useful-inequality-inequality}
\bE\|y^*(x_{s-1}) - y_{s}\|_2^2 \leq \Bigl(1 - 1/(12\kappa)\Bigr)\bE \|y^*(x_{s-1})-y_{s-1}\|_2^2 +\varrho(\mu_2),
\ee
where $\varrho(\mu_2)=\mu_2^2 d_2/6 + \mu_2^2(d_2+6)^3/36$.
\end{lemma}

\begin{proof}
According to the updates in Algorithm \ref{ZO-GDA}, we have
\[\ba{lll}
\|y^*(x_{s-1})-y_s\|^2 & = & \|\Proj_{\mathcal{Y}}(y_{s-1} + \eta_2 H_{\mu_2}(x_{s-1},y_{s-1},\bu_{2,[q_2]}) - y^*(x_{s-1}))\|_2^2\\
                       & \leq & \|y^*(x_{s-1})-y_{s-1}\|^2 +2\eta_2\langle H_{\mu_2}(x_{s-1},y_{s-1},\bu_{2,[q_2]}), y_{s-1} - y^*({x_{s-1}})\rangle \\ && + \eta_2^2\|H_{\mu_2}(x_{s-1},y_{s-1},\bu_{2,[q_2]})\|_2^2.
\ea\]
For a given $s$, we use $\bE$ to denote the expectation with respect to random samples $\bu_{2,[q_2]}$ conditioned on all previous iterations. By taking expectation to both sides of the above inequality, we obtain
\[\ba{ll}
   & \bE \|y^*(x_{s-1})-y_s\|^2 \\
\leq & \bE \|y^*(x_{s-1})-y_{s-1}\|^2 - 2\eta_2 \langle -\nabla_y f_{\mu_2}(x_{s-1}, y_{s-1}), y_{s-1} - y^*({x_{s-1}})\rangle + \eta_2^2 \bE\|H_{\mu_2}(x_{s-1},y_{s-1},\bu_{2,[q_2]})\|_2^2\\
\leq & \bE \|y^*(x_{s-1})-y_{s-1}\|^2 - 2\eta_2[f_{\mu_2}(x_{s-1},y^*({x_{s-1}})) - f_{\mu_2}(x_{s-1},y_{s-1})] \\
     & + \eta_2^2 \Bigl(3\|\nabla_y f(x_{s-1},y_{s-1})\|_2^2 + \mu_2^2\ell^2 (d_2+6)^3\Bigr) \\
\leq &\bE \|y^*(x_{s-1})-y_{s-1}\|^2 - 2\eta_2(f(x_{s-1},y^*({x_{s-1}})) - f(x_{s-1},y_{s-1})) + \mu_2^2 d_2 \eta_2 \ell\\
     & + \eta_2^2(6\ell(f(x_{s-1},y^*({x_{s-1}}))-f(x_{s-1},y_{s-1}))+\eta_2^2\mu_2^2\ell^2 (d_2+6)^3 \\
= & \bE \|y^*(x_{s-1})-y_{s-1}\|^2 - (f(x_{s-1},y^*({x_{s-1}}))-f(x_{s-1},y_{s-1}))/(6\ell) + \varrho(\mu_2)\\
\leq & \bE \|y^*(x_{s-1})-y_{s-1}\|^2\Bigl(1-\frac{\tau}{12\ell}\Bigr) + \varrho(\mu_2),
\ea\]
where the second inequality is due to the concavity of $f_{\mu_2}(x_{s-1},\cdot)$ (see Lemma \ref{lemma:f-mu-convex}) and Lemma \ref{lemma:Averaged Upper bound}, the third inequality is due to Lemma \ref{lemma:smooth-stronglyconvex} and Lemma \ref{lemma:NS17-Eq19}, the equality is due to $\eta_2 = 1/(6\ell)$, and the last inequality is due to Lemma \ref{lemma:smooth-stronglyconvex}. This completes the proof.
\end{proof}

We now prove the following upper bound of $\bE \|y_s - y^*(x_s)\|_2^2$.

\begin{lemma}\label{lemma:upper bound ZO GDA}
Consider \texttt{ZO-GDA} (Algorithm \ref{ZO-GDA}). Use the same notation and the same assumptions as in Lemma \ref{lemma:ZO-GDA-useful-inequality}. Denote $\delta_s = \|y_s - y^*(x_s)\|_2^2$ and set $\eta_1$ as in \eqref{eta1}, and
\be\label{gamma}
\gamma := 1- \frac{1}{24\kappa} + 144\ell^2\kappa^3\eta_1^2\leq 1-\frac{5}{144\kappa}<1.
\ee
It holds that
\be\label{lemma:upper bound ZO GDA-inequality}
\bE \delta_s \leq \gamma^s \bE \delta_0 + \alpha_1 \sum_{i=0}^{s-1} \gamma^{s-1-i}\bE\|\nabla g(x_{i-1})\|_2^2 + \theta_0\sum_{i=0}^{s-1}\gamma^{s-1-i},
\ee
where
\be\label{alpha1}\alpha_1 = \frac{9}{12^8\kappa(\kappa+1)^4(\ell+1)^2}, \ \theta_0 = \alpha_2 \mu_1^2(d_1+6)^3 + 2\varrho(\mu_2), \ \alpha_2 = \frac{1}{8\times 12^7\kappa(\kappa+1)^4}.
\ee
\end{lemma}

\begin{proof}
Define the filtration $\mathcal{F}_s = \Bigl\{x_s, y_s,x_{s-1},y_{s-1},...,x_{1},y_{1}\Bigr\}$. Let $\zeta_s = (\bu_{1,i\in [q_1]},\bu_{2,i\in[q_2]}),\zeta_{[s]} = (\zeta_1,\zeta_2,...,\zeta_s)$. Denote by $\bE$ taking expectation w.r.t $\zeta_{[s]}$ conditioned on $\mathcal{F}_{s}$ and then taking expectation over $\mathcal{F}_{s}$. Since $\kappa > 1$, using the Young's inequality, we have
\be\label{lemma:upper bound ZO GDA-proof-1}\ba{ll}
\bE \delta_s & = \bE\|y^*(x_s) - y_s\|_2^2 \\
             & \leq \Bigl(1 + \frac{1}{2(12\kappa-1)}\Bigr)\bE\|y^*(x_{s-1}) - y_{s}\|_2^2 + \Bigl(1 + 2(12\kappa-1)\Bigr)\bE\|y^*(x_s) - y^*(x_{s-1})\|_2^2\\
             & \leq (1-\frac{1}{24\kappa-1})(1-\frac{1}{12\kappa})\bE\|y^*(x_{s-1}) - y_s\|_2^2 + 24\kappa \bE\|y^*(x_s) - y^*(x_{s-1})\|_2^2+2\varrho(\mu_2)\\
             & \leq (1 - \frac{1}{24\kappa})\bE \|y^*(x_{s-1}) - y_{s-1}\|_2^2 + 24\kappa^3\bE\|x_s - x_{s-1}\|_2^2 + 2\varrho(\mu_2)\\
             & = (1-\frac{1}{24\kappa})\bE\delta_{s-1}+24\kappa^3\eta_1^2\bE\|G_{\mu_1}(x_{s-1},y_{s-1},\bu_{1,[q_1]})\|_2^2+ 2\varrho(\mu_2) \\
             & = (1-\frac{1}{24\kappa})\bE\delta_{s-1}+\frac{\alpha_1}{6}\bE\|G_{\mu_1}(x_{s-1},y_{s-1},\bu_{1,[q_1]})\|_2^2+ 2\varrho(\mu_2),
\ea\ee
where
the second inequality is due to \eqref{lemma:ZO-GDA-useful-inequality-inequality}, the third inequality is due to Lemma \ref{lemma:smoothed convexity}. From Lemma \ref{lemma:Averaged Upper bound}, we have
\be\label{g upper bound}\ba{ll}
& \bE_{\bu_{1,[q_1]}} \|G_{\mu_1}(x_{s-1},y_{s-1},\bu_{1,[q_1]})\|_2^2 \\
\leq & 3\bE\|\nabla_x f(x_{s-1},y_{s-1})\|_2^2 + \mu_1^2\ell^2 (d_1+6)^3\\
\leq & 6\bE\|\nabla g(x_{s-1})\|_2^2 + 6\ell^2\bE\|y^*(x_{s-1}) - y_{s-1}\|_2^2 +\mu_1^2\ell^2 (d_1+6)^3,
\ea\ee
where the second inequality is due to Assumption \ref{assumption:A1-A2-A4}. Combining \eqref{lemma:upper bound ZO GDA-proof-1} and \eqref{g upper bound} yields \eqref{lemma:upper bound ZO GDA-inequality} by noting \eqref{gamma}.
\end{proof}

Now we are ready to prove Theorem \ref{ZO_GDA_thm}.

\begin{proof}{({\bf Proof of Theorem \ref{ZO_GDA_thm}.})}
First, the following inequalities hold:
\[\ba{ll}
     & g(x_{s+1}) \\
\leq & g(x_s) - \eta_1 \langle \nabla g(x_{s}),G_{\mu_1}(x_s,y_{s},\bu_{1,[q_1]})\rangle + \frac{1}{2}L_g \eta_1^2 \|G_{\mu_1}(x_s,y_{s},\bu_{1,[q_1]})\|_2^2\\
=    & g(x_s)-\eta_1\Bigl\langle \nabla_x f(x_s,y^*(x_s)) - \nabla_x f_{\mu_1}(x_s,y^*(x_s))+ \nabla_x f_{\mu_1}(x_s,y^*(x_s)) - \nabla_x f_{\mu_1}(x_s,y_{s})\\
& +\nabla_x f_{\mu_1}(x_s,y_{s}), G_{\mu_1}(x_s,y_{s},\bu_{1,[q_1]})\Bigr\rangle + \frac{1}{2}L_g \eta_1^2 \|G_{\mu_1}(x_s,y_{s},\bu_{1,[q_1]})\|_2^2\\
\leq & g(x_s) + \|\nabla_x f(x_s,y^*(x_s)) - \nabla_x f_{\mu_1}(x_s,y^*(x_s))\|^2/L_g + \frac{L_g\eta_1^2}{4}\|G_{\mu_1}(x_s,y_{s},\bu_{1,[q_1]})\|^2 \\
     & + \|\nabla_x f_{\mu_1}(x_s,y^*(x_s)) - \nabla_x f_{\mu_1}(x_s,y_{s})\|^2/L_g + \frac{L_g\eta_1^2}{4}\|G_{\mu_1}(x_s,y_{s},\bu_{1,[q_1]})\|^2 \\
     & -\eta_1\langle\nabla_x f_{\mu_1}(x_s,y_{s}), G_{\mu_1}(x_s,y_{s},\bu_{1,[q_1]})\rangle + \frac{1}{2}L_g \eta_1^2 \|G_{\mu_1}(x_s,y_{s},\bu_{1,[q_1]})\|_2^2 \\
\leq & g(x_s) + \frac{\ell^2}{L_g}\|y^*(x_s) - y_{s}\|_2^2  - \eta_1 \langle \nabla_x f_{\mu_1}(x_s,y_{s}), G_{\mu_1}(x_s,y_{s},\bu_{1,[q_1]})\rangle \\
& + \eta_1^2L_g \|G_{\mu_1}(x_s,y_{s},\bu_{1,[q_1]})\|_2^2 + \frac{\mu_1^2}{4L_g}\ell^2(d_1 +3)^{3},
\ea\]
where the first inequality is due to Lemma \ref{lemma:smoothed convexity} and the Descent lemma, the second inequality is due to Young's inequality, and the last inequality is due to Lemmas \ref{lemma:diff-f-fmu} and \ref{lemma:smooth of smoothed function}. Now take expectation with respect to $\bu_{1,[q_1]}$ to the above inequality, we get:
\be\label{proof-thm-ZO-GDA-inequality-1}
\ba{lll}
    \eta_1 \bE \|\nabla_x f_{\mu_1}(x_s,y_{s})\|_2^2 &\leq &\bE g(x_s) - \bE g(x_{s+1}) + \frac{\ell^2}{L_g}\bE\|y^*(x_s)-y_{s}\|_2^2 \\
    &&+ \eta_1^2 L_g \bE\|G_{\mu_1}(x_s,y_{s},\bu_{1,[q_1]})\|_2^2 + \frac{\mu_1^2}{4L_g}\ell^2(d_1 +3)^3.
\ea\ee
From Lemma \ref{lemma:smooth of smoothed function}, we have
\be\label{proof-thm-ZO-GDA-inequality-2}
    \eta_1 \bE\|\nabla_x f_{\mu_1}(x_s,y^*(x_s))\|_2^2 \leq 2\eta_1\bE\|\nabla_x f_{\mu_1}(x_s,y_s)\|_2^2 + 2\eta_1\ell^2 \|y_s - y^*(x_s)\|_2^2.
\ee
From Lemma \ref{lemma:diff-f-fmu}, we have
\be\label{proof-thm-ZO-GDA-inequality-3}
    \eta_1\|\nabla g(x_s)\|_2^2 \leq 2\eta_1\|\nabla_x f_{\mu_1}(x_s,y^*(x_s))\|_2^2 + \frac{\eta_1 \mu_1^2}{2}\ell^2 (d_1+3)^3.
\ee
Combining \eqref{g upper bound}, \eqref{proof-thm-ZO-GDA-inequality-1}, \eqref{proof-thm-ZO-GDA-inequality-2}, \eqref{proof-thm-ZO-GDA-inequality-3} yields,
\be\label{proof-thm-ZO-GDA-inequality-4}\ba{ll}
     &\eta_1 \bE\|\nabla g(x_s)\|_2^2 \\
\leq &4\bE g(x_s) - 4\bE g(x_{s+1}) + \Bigl(\frac{4\ell^2}{L_g } + 4\eta_1\ell^2\Bigr)\bE\|y^*(x_s)-y_{s}\|_2^2+ \frac{\mu_1^2}{L_g}\ell^2(d_1 +3)^3 \\ & + \frac{\eta_1\mu_1^2}{2}\ell^2(d_1+3)^3 
     + 4\eta_1^2L_g\Bigl[6\bE\|\nabla g(x_{s})\|_2^2 + 6\ell^2\bE\|y^*(x_{s}) - y_{s}\|_2^2+\mu_1^2\ell^2(d_1+6)^3\Bigr]\\
=   & 4\bE g(x_s) - 4\bE g(x_{s+1}) +24\eta_1^2L_g \bE\|\nabla g(x_{s})\|_2^2+ \theta_1\bE\delta_s + \theta_2,
\ea\ee
where
\be\label{theta-1}
\theta_1=\frac{4 \ell^2}{L_g } + 4\eta_1 \ell^2+24\eta_1^2L_g\ell^2\leq 4\ell+4\eta_1 \ell^2+24\eta_1^2\ell^3(\kappa+1),
\ee
and
\begin{align}\label{theta-2}
\theta_2 =~& \frac{\mu_1^2}{L_g}\ell^2(d_1 +3)^3 + \frac{\eta_1 \mu_1^2}{2}\ell^2(d_1+3)^3 + 4\eta_1^2L_g\mu_1^2\ell^2(d_1+6)^3 \nonumber\\
\leq~&\mu_1^2\ell(d_1 +3)^3 + \frac{\eta_1 \mu_1^2}{2}\ell^2(d_1+3)^3 + 4\eta_1^2(\kappa+1)\ell^3\mu_1^2(d_1+6)^3,
\end{align}
where we have used the definition of $L_g:=\ell(\kappa+1)$. Taking sum over $s=0,\ldots,S$ to both sides of \eqref{lemma:upper bound ZO GDA-inequality}, we get
\be\label{proof-thm-ZO-GDA-inequality-5}
\sum_{s=0}^S\bE \delta_s \leq \sum_{s=0}^S\gamma^s \bE \delta_0 + \alpha_1 \sum_{s=0}^S\sum_{i=0}^{s-1} \gamma^{s-1-i}\bE\|\nabla g(x_{i-1})\|_2^2 + \theta_0\sum_{s=0}^S\sum_{i=0}^{s-1}\gamma^{s-1-i}.
\ee
Moreover, from \eqref{gamma} it is easy to obtain
\be\label{proof-thm-ZO-GDA-inequality-6}
\sum_{s=0}^S \gamma^s \leq 36\kappa, \quad \sum_{s=0}^S\sum_{i=0}^{s-1}\gamma^{s-1-i} \leq 36\kappa(S+1),
\ee
and
\be\label{proof-thm-ZO-GDA-inequality-7}
\sum_{s=0}^S\sum_{i=0}^{s-1} \gamma^{s-1-i}\bE\|\nabla g(x_{i-1})\|_2^2 \leq 36\kappa\sum_{s=0}^S\bE\|\nabla g(x_{s})\|_2^2.
\ee
Substituting \eqref{proof-thm-ZO-GDA-inequality-6} and \eqref{proof-thm-ZO-GDA-inequality-7} into \eqref{proof-thm-ZO-GDA-inequality-5}, we obtain
\be\label{proof-thm-ZO-GDA-inequality-8}
\sum_{s=0}^S\bE \delta_s \leq 36\kappa\bE \delta_0 + 36\kappa\alpha_1 \sum_{s=0}^S\bE\|\nabla g(x_{s})\|_2^2 + 36\kappa\theta_0(S+1).
\ee
Now, summing \eqref{proof-thm-ZO-GDA-inequality-4} over $s=0,\ldots,S$ yields
\be\label{proof-thm-ZO-GDA-inequality-9}\ba{ll}
     &\eta_1 \sum_{s=0}^S\bE\|\nabla g(x_s)\|_2^2 \\
=    & 4\bE g(x_0) - 4\bE g(x_{S+1}) +24\eta_1^2L_g \sum_{s=0}^S\bE\|\nabla g(x_{s})\|_2^2+ \theta_1\sum_{s=0}^S\bE\delta_s + (S+1)\theta_2 \\
\leq & 4\bE g(x_0) - 4\bE g(x_{S+1}) +24\eta_1^2L_g \sum_{s=0}^S\bE\|\nabla g(x_{s})\|_2^2 \\
     & + \theta_1 [36\kappa\bE \delta_0 + 36\kappa\alpha_1 \sum_{s=0}^S\bE\|\nabla g(x_{s})\|_2^2 + 36\kappa\theta_0(S+1)] + (S+1)\theta_2
\ea\ee
where the second inequality is from \eqref{proof-thm-ZO-GDA-inequality-8}. Using \eqref{theta-1}, \eqref{alpha1} and \eqref{eta1}, it is easy to verify that
\[36\kappa\theta_1\alpha_1 \leq \left(\frac{108}{3\times 12^3} + \frac{108}{12^7}+\frac{54}{4\times 12^{10}}\right)\eta_1 \leq 0.021\eta_1,\]
which together with $L_g:=(\kappa+1)\ell$ yields
\be\label{coefficient-1}
36\kappa\theta_1\alpha_1 + 24\eta_1^2L_g \leq 0.021\eta_1 + 0.0003\eta_1 = 0.0213\eta_1.
\ee
Combining \eqref{proof-thm-ZO-GDA-inequality-9} and \eqref{coefficient-1} yields
\be\label{proof-thm-ZO-GDA-inequality-10}\ba{ll}
     & 0.9787 \eta_1 \sum_{s=0}^S\bE\|\nabla g(x_s)\|_2^2 \\
\leq & 4\bE g(x_0) - 4\bE g(x_{S+1}) + \theta_1 [36\kappa\bE \delta_0 + 36\kappa\theta_0(S+1)] + (S+1)\theta_2.
\ea\ee
Dividing both sides of \eqref{proof-thm-ZO-GDA-inequality-10} by $0.9787 \eta_1(S+1)$ yields
\be\label{proof-thm-ZO-GDA-inequality-10}\ba{ll}
 \frac{1}{S+1} \sum_{s=0}^S\bE\|\nabla g(x_s)\|_2^2
\leq \frac{4\Delta_g}{0.9787 \eta_1(S+1)} + \frac{36\kappa\theta_1 \bE \delta_0}{0.9787 \eta_1(S+1)} + \frac{36\kappa\theta_1 \theta_0}{0.9787 \eta_1} + \frac{\theta_2}{0.9787 \eta_1},
\ea\ee
where $\Delta_g := g(x_0)-\min_{x\in\br^{d_1}} g(x)$. Now we only need to upper bound the right hand side of \eqref{proof-thm-ZO-GDA-inequality-10} by $\epsilon^2$, and this can be guaranteed by choosing the parameters as in \eqref{thm3.1-S-mu1-mu2}. This completes the proof of Theorem \ref{ZO_GDA_thm}.
\end{proof}

\begin{remark}\label{remark:cversusuc}
Note that the term $\delta_0$ appearing in~\eqref{proof-thm-ZO-GDA-inequality-10} is defined as $\delta_0:=\| y_0 - y^*(x_0)\|_2^2$. Under the assumption that the set $\mathcal{Y}$ is bounded, this term could be upper bounded by $D^2$. This is the only place in the proof where we require the constraint set $\mathcal{Y}$ to be bounded. In the unconstrained case, when $\mathcal{Y}:=\mathbb{R}^{d_2}$, having $\delta_0$ being bounded away from infinity is dependent on the initial values $(x_0,y_0)$ supplied to the algorithm. 
{In fact, by defining $h(y) := f(x_0,y)$, we know that
$y^*(x_0) = \argmax_{y\in\mathcal{Y}} h(y)$. Since $f(x,\cdot)$ is $\tau$-strongly concave for all $x\in\mathbb{R}^{d_1}$, we know that $h(y)$ is $\tau$-strongly concave. By Lemma \ref{lemma:smooth-stronglyconvex}, we have 
\[\|y_0-y^*(x_0)\|\leq\frac{1}{\tau}\|\nabla h(y_0)\| = \frac{1}{\tau}\|\nabla_y f(x_0,y_0)\|.\]
Therefore, $\delta_0$ is upper bounded by a constant depending only on $x_0$ and $y_0$.
}
Indeed this scenario is common in the complexity analysis of optimization algorithms~\cite{nesterov2018lectures}.
\end{remark}

\section{Convergence analysis of \texttt{ZO-GDMSA} (Algorithm \ref{ZO-GDMSA})}

First, we show the following iteration complexity of the inner loop for $y$ in Algorithm \ref{ZO-GDMSA}.
\begin{lemma}\label{GDmax_linear_convergence}
In Algorithm \ref{ZO-GDMSA}, setting $\eta_2=1/(6\ell)$, $\mu_2=\cO(\kappa^{-1/2}d_2^{-3/2}\epsilon)$ and $T= \mathcal{O}(\kappa  \log(\epsilon^{-1}))$. For fixed $x_s$ in the $s$-th iteration, we have $\bE\|y^*(x_s)-y_{T}(x_s)\|_2^2 \leq \epsilon^2$.
\end{lemma}

\begin{proof}
According to the updates in Algorithm \ref{ZO-GDMSA}, we have
\[\ba{ll}
  & \|y^*(x_s)-y_{t+1}(x_s)\|^2 \\
= &(\|\Proj_{\mathcal{Y}}(y_t(x_s) + \eta_2 H_{\mu_2}(x_s,y_{t}(x_s)),\bu_{2,[q_2]} - y^*(x_s))\|_2^2)\\
\leq & \|y^*(x_s)-y_t(x_s)\|^2 +2\eta_2\langle H_{\mu_2}(x_s, y_{t}(x_s),\bu_{2,[q_2]},y_t(x_s) - y^*(x_s)\rangle + \eta_2^2\|H_{\mu_2}(x_s, y_{t}(x_s),\bu_{2,[q_2]}\|_2^2.
\ea\]
For a given $s$, denote by $\bE$ taking expectation with respect to random samples $\bu_{2,[q_2]}$ conditioned on all previous iterations. By taking expectation to both sides of this inequality, we obtain
\[\ba{ll}
     & \bE \|y^*(x_s)-y_{t+1}(x_s)\|^2 \\
\leq & \bE \|y^*(x_s)-y_t(x_s)\|^2 - 2\eta_2 \langle -\nabla_y f_{\mu_2}(x_s, y_{t}(x_s)),y_t(x_s) - y^*(x_s)\rangle + \eta_2^2 \bE\|H_{\mu_2}(x_s, y_{t}(x_s),\bu_{2,[q_2]})\|_2^2\\
\leq & \bE \|y^*(x_s)-y_t(x_s)\|^2 - 2\eta_2 \langle -\nabla_y f_{\mu_2}(x_s, y_{t}(x_s)),y_t(x_s) - y^*(x_s)\rangle + \eta_2^2 \Bigl(3\|\nabla_y f(x_s,y_t(x_s))\|_2^2 + \mu_2^2\ell^2 (d_2+6)^3\Bigr)\\
\leq & \bE \|y^*(x_s)-y_t(x_s)\|^2 - 2\eta_2[f_{\mu_2}(x_s,y^*(x_s)) - f_{\mu_2}(x_s,y_t(x_s))] + \eta_2^2 \Bigl(3\|\nabla_y f(x_s,y_t(x_s))\|_2^2 + \mu_2^2\ell^2 (d_2+6)^3\Bigr) \\
\leq & \bE \|y^*(x_s)-y_t(x_s)\|^2 - 2\eta_2(f(x_s,y^*(x_s)) - f(x_s,y_{t}(x_s))) + 2\mu_2^2 d_2 \eta_2 \ell+ \eta_2^2(6L_2(f(x_s,y^*(x_s))-f(x_s,y_{t}(x_s)))\\
&+\eta_2^2\mu_2^2\ell^2(d_2+6)^3\\ 
= & \bE \|y^*(x_s)-y_t(x_s)\|^2 - (f(x_s,y^*(x_s))-f(x_s,y_{t}(x_s)))/(6\ell) + \mu_2^2 d_2/3 + \mu_2^2(d_2+6)^3/36\\
\leq & \bE \|y^*(x_s)-y_t(x_s)\|^2 \Bigl(1-\frac{\tau}{12\ell}\Bigr) + \mu_2^2 d_2/3 + \mu_2^2(d_2+6)^3/36,
\ea\]
where the second inequality is due to Lemma \ref{lemma:Averaged Upper bound}, the third inequality is due to the concavity of $f_{\mu_2}(x_s,\cdot)$ (see Lemma \ref{lemma:f-mu-convex}), the fourth inequality is due to Lemmas \ref{lemma:NS17-Eq19} and \ref{lemma:smooth-stronglyconvex}, the equality is due to $\eta_2=1/(6\ell)$, and the last inequality is due to Lemma \ref{lemma:smooth-stronglyconvex}.

Define $\delta = 12\ell(\mu_2^2 d_2/3 + \mu_2^2(d_2+6)^3/36)/\tau$. From the above inequality, we have
\[\ba{ll}
\bE \|y^*(x_s)-y_t(x_s)\|^2-\delta&\leq (\bE \|y^*(x_s)-y_{t-1}(x_s)\|^2-\delta)\Bigl(1-\frac{\tau}{12\ell}\Bigr) \\
& \leq(\bE \|y^*(x_s)-y_0(x_s)\|^2 - \delta)\Bigl(1-\frac{\tau}{12\ell}\Bigr)^{t}\\
& \leq \bE \|y^*(x_s)-y_0(x_s)\|^2\Bigl(1-\frac{\tau}{12\ell}\Bigr)^{t} \leq D^2 \Bigl(1 - \frac{\tau}{12\ell}\Bigr)^t,
\ea\]
where the last inequality is due to Assumption \ref{assumption:A1-A2-A4}. Now it is clear that in order to ensure that $\bE \|y^*(x_s)-y_T(x_s)\|^2 \leq \epsilon^2$, we need $T = \cO(\kappa  \log(\epsilon^{-1}))$ and $\mu_2=\cO(\kappa^{-1/2}d_2^{-3/2}\epsilon)$.
\end{proof}

We are now ready to prove Theorem \ref{ZO_GDmax_thm}. 

\begin{proof}{({\bf Proof of Theorem \ref{ZO_GDmax_thm}.})}
First, the following inequalities hold:
\[\ba{ll}
     & g(x_{s+1}) \\
\leq & g(x_s)-\eta_1\langle \nabla_x g(x_{s}),G_{\mu_1}(x_s,y_{s+1},\bu_{1,[q_1]})\rangle +\frac{1}{2}L_g\eta_1^2\|G_{\mu_1}(x_s,y_{s+1},\bu_{1,[q_1]})\|_2^2\\
= & g(x_s)-\eta_1\Bigl\langle \nabla_x f(x_s,y^*(x_s)) - \nabla_x f_{\mu_1}(x_s,y^*(x_s))+ \nabla_x f_{\mu_1}(x_s,y^*(x_s)) - \nabla_x f_{\mu_1}(x_s,y_{s+1})\\ &~+\nabla_x f_{\mu_1}(x_s,y_{s+1}),G_{\mu_1}(x_s,y_{s+1},\bu_{1,[q_1]})\Bigr\rangle + \frac{1}{2}L_g \eta_1^2 \|G_{\mu_1}(x_s,y_{s+1},\bu_{1,[q_1]})\|_2^2\\
\leq & g(x_s) + \|\nabla_x f(x_s,y^*(x_s)) - \nabla_x f_{\mu_1}(x_s,y^*(x_s))\|^2/L_g + \frac{L_g\eta_1^2}{4}\|G_{\mu_1}(x_s,y_{s+1},\bu_{1,[q_1]})\|^2 \\
     & + \|\nabla_x f_{\mu_1}(x_s,y^*(x_s)) - \nabla_x f_{\mu_1}(x_s,y_{s+1})\|^2/L_g + \frac{L_g\eta_1^2}{4}\|G_{\mu_1}(x_s,y_{s+1},\bu_{1,[q_1]})\|^2 \\
     & -\eta_1\langle\nabla_x f_{\mu_1}(x_s,y_{s+1}), G_{\mu_1}(x_s,y_{s+1},\bu_{1,[q_1]})\rangle + \frac{1}{2}L_g \eta_1^2 \|G_{\mu_1}(x_s,y_{s+1},\bu_{1,[q_1]})\|_2^2 \\
\leq & g(x_s) + \frac{\ell^2}{L_g}\|y^*(x_s) - y_{s+1}\|_2^2  - \eta_1 \langle \nabla_x f_{\mu_1}(x_s,y_{s+1}), G_{\mu_1}(x_s,y_{s+1},\bu_{1,[q_1]})\rangle \\
& + \eta_1^2L_g \|G_{\mu_1}(x_s,y_{s+1},\bu_{1,[q_1]})\|_2^2 + \frac{\mu_1^2}{4L_g}\ell^2(d_1 +3)^{3},
\ea\]
where the first inequality is due to Lemma \ref{lemma:smoothed convexity}, the second inequality is due to Young's inequality, and the last inequality is due to Lemmas \ref{lemma:diff-f-fmu} and \ref{lemma:smooth of smoothed function}. Now take expectation with respect to $\bu_{1,[q_1]}$ to the above inequality, we get:
\be\label{proof-thm-ZO-GDMSA-inequality-1}\ba{ll}
& \eta_1 \bE\|\nabla_x f_{\mu_1}(x_s,y_{s+1})\|_2^2 \\
\leq & \bE g(x_s) -\bE g(x_{s+1}) + \frac{\ell^2}{L_g }\bE\|y^*(x_s)-y_{s+1}\|_2^2 + \eta_1^2 L_g \bE\|G_{\mu_1}(x_s,y_{s+1},\bu_{1,[q_1]})\|_2^2 + \frac{\mu_1^2}{4L_g}\ell^2(d_1 +3)^3 \\
\leq & \bE g(x_s) -\bE g(x_{s+1}) + \frac{\ell^2}{L_g }\bE\|y^*(x_s)-y_{s+1}\|_2^2 + \eta_1^2 L_g \Bigl(3\|\nabla_x f(x_s,y_{s+1})\|_2^2 + \mu_1^2\ell^2 (d_1+6)^3\Bigr)+ \frac{\mu_1^2}{4L_g}\ell^2(d_1 +3)^3,
\ea\ee
where the second inequality is due to Lemma \ref{lemma:Averaged Upper bound}. From Lemma \ref{lemma:diff-f-fmu} we have 
\be\label{proof-thm-ZO-GDMSA-inequality-2}\ba{ll}
\bE\|\nabla_x f(x_s,y_{s+1})\|_2^2 \leq 2\bE\|\nabla_x f_{\mu_1}(x_s,y_{s+1})\|_2^2 + \mu_1^2\ell^2(d_1+3)^3/2.
\ea\ee
Combining \eqref{proof-thm-ZO-GDMSA-inequality-1} and \eqref{proof-thm-ZO-GDMSA-inequality-2}, and noting $\eta_1=1/(12L_g)$, we have
\be\label{proof-thm-ZO-GDMSA-inequality-3}\ba{ll}
\bE\|\nabla_x f(x_s,y_{s+1})\|_2^2\leq & 48L_g \Bigl[\bE g(x_s) -\bE g(x_{s+1})\Bigr] + 48 \ell^2\bE\|y^*(x_s)-y_{s+1}\|_2^2\\
&+ 13\mu_1^2 \ell^2 (d_1+3)^3 + \mu_1^2\ell^2 (d_1+6)^3/3.
\ea\ee
It then follows that 
\be\label{proof-thm-ZO-GDMSA-inequality-4}\ba{ll}
& \bE\|\nabla g(x_s)\|_2^2 \\
\leq & 2 \bE\|\nabla_x g(x_s) - \nabla_x f(x_s,y_{s+1})\|_2^2 + 2\bE\|\nabla_x f(x_s,y_{s+1})\|_2^2 \\
\leq & 2\ell^2 \bE\|y^*(x_s) - y_{s+1}\|_2^2 + 2\bE\|\nabla_x f(x_s,y_{s+1})\|_2^2 \\
\leq & 96L_g \Bigl[\bE g(x_s) -\bE g(x_{s+1})\Bigr] + 98\ell^2\bE\|y^*(x_s)-y_{s+1}\|_2^2\\
&+ 26\mu_1^2 \ell^2 (d_1+3)^3+ 2\mu_1^2\ell^2 (d_1+6)^3/3,
\ea\ee
where the second inequality is due to Assumption \ref{assumption:A1-A2-A4}, and the last inequality is due to \eqref{proof-thm-ZO-GDMSA-inequality-3}.

Take the sum over $s=0,\ldots,S$ to both sides of \eqref{proof-thm-ZO-GDMSA-inequality-4}, we get
\be\label{proof-thm-ZO-GDMSA-inequality-5}\ba{ll}
\frac{1}{S+1}\sum_{s=0}^S \bE \|\nabla g(x_s)\|_2^2 \leq & \frac{96L_g}{S+1}\bE [g(x_0) - g(x_{S+1})] + \frac{98\ell^2}{S+1}\sum_{s=0}^S\bE\|y^*(x_s)-y_{s+1}\|_2^2\\
&+ 26\mu_1^2 \ell^2(d_1+3)^3+2\mu_1^2\ell^2 (d_1+6)^3/3.
\ea\ee
Denote $\Delta_g = g(x_0) - \min_{x\in\mathbb{R}^{d_1}}(g(x))$. From Lemma \ref{GDmax_linear_convergence}, we know that when $T = \cO (\kappa \log(\epsilon^{-1}))$, we have $\bE\|y^*(x_s)-y_{s+1}\|^2\leq\epsilon^2$ (note that $y_{s+1}=y_T(x_s)$). Therefore, choosing 
parameters as in \eqref{thm3.2-S-mu1-mu2} guarantees that the right hand side of \eqref{proof-thm-ZO-GDMSA-inequality-5} is upper bounded by $\mathcal{O}(\epsilon^2)$, and thus an $\epsilon$-stationary point is found.
This completes the proof. 
\end{proof}

\section{Convergence analysis for \texttt{ZO-SGDA} (Algorithm \ref{ZO-SGDA})}
We first show the following inequality.

\begin{lemma}\label{lemma:ZO-SGDA-useful-inequality}
Assume $\{(x_s,y_s)\}$ is the sequence generated by Algorithm \ref{ZO-SGDA}. By setting $\eta_2 = 1/(6\ell)$, the following inequality holds:
\be\label{lemma:ZO-SGDA-useful-inequality-inequality}
\bE\|y^*(x_{s-1}) - y_{s}\|_2^2 \leq \Bigl(1 - 1/(12\kappa)\Bigr)\bE \|y^*(x_{s-1})-y_{s-1}\|_2^2 +\varrho(\mu_2,\epsilon),
\ee
where $\varrho(\mu_2,\epsilon)=\mu_2^2d_2/3 + \mu_2^2 (d_2+3)^2/72 + \mu_2^2 (d_2+6)^2\epsilon^2/576 + \epsilon^2/72\ell^2$.
\end{lemma}

\begin{proof}
According to the updates in Algorithm \ref{ZO-SGDA}, we have
\[\ba{lll}
\|y^*(x_{s-1})-y_s\|^2 & = & \|\Proj_{\mathcal{Y}}(y_{s-1} + \eta_2 H_{\mu_2}(x_{s-1},y_{s-1},\bu_{\cM_2},\xi_{\cM_2}) - y^*(x_{s-1}))\|_2^2\\
                       & \leq & \|y^*(x_{s-1})-y_{s-1}\|^2 +2\eta_2\langle H_{\mu_2}(x_{s-1},y_{s-1},\bu_{\cM_2},\xi_{\cM_2}), y_{s-1} - y^*({x_{s-1}})\rangle \\ && + \eta_2^2\|H_{\mu_2}(x_{s-1},y_{s-1},\bu_{\cM_2},\xi_{\cM_2})\|_2^2.
\ea\]
For a given $s$, denote by $\bE$ taking expectation with respect to random samples $\bu_{\cM_2},\xi_{\cM_2}$ conditioned on all previous iterations. By taking expectation to both sides of this inequality, we obtain
\[\ba{ll}
   & \bE \|y^*(x_{s-1})-y_s\|^2 \\
\leq & \bE \|y^*(x_{s-1})-y_{s-1}\|^2 - 2\eta_2 \langle -\nabla_y f_{\mu_2}(x_{s-1}, y_{s-1}), y_{s-1} - y^*({x_{s-1}})\rangle + \eta_2^2 \bE\|H_{\mu_2}(x_{s-1},y_{s-1},\bu_{\cM_2},\xi_{\cM_2})\|_2^2\\
\leq & \bE \|y^*(x_{s-1})-y_{s-1}\|^2 - 2\eta_2[f_{\mu_2}(x_{s-1},y^*({x_{s-1}})) - f_{\mu_2}(x_{s-1},y_{s-1})] \\
     & + \eta_2^2 \Bigl(3\|\nabla_y f(x_{s-1},y_{s-1})\|_2^2 + \epsilon(\mu_2)\Bigr) \\
\leq &\bE \|y^*(x_{s-1})-y_{s-1}\|^2 - 2\eta_2(f(x_{s-1},y^*({x_{s-1}})) - f(x_{s-1},y_{s-1})) + 2\mu_2^2 d_2 \eta_2 \ell\\
     & + \eta_2^2(6\ell(f(x_{s-1},y^*({x_{s-1}}))-f(x_{s-1},y_{s-1}))+\eta_2^2\varrho_2(\epsilon,\mu_2)\\
= & \bE \|y^*(x_{s-1})-y_{s-1}\|^2 - (f(x_{s-1},y^*({x_{s-1}}))-f(x_{s-1},y_{s-1}))/(6\ell) + \varrho(\mu_2,\epsilon)\\
\leq & \bE \|y^*(x_{s-1})-y_{s-1}\|^2\Bigl(1-\frac{\tau}{12\ell}\Bigr) + \varrho(\mu_2,\epsilon),
\ea\]
where the second inequality is due to the concavity of $f_{\mu_2}(x_{s-1},\cdot)$ (see Lemma \ref{lemma:f-mu-convex}) and Lemma \ref{lemma:mini batch upper bound}, the third inequality is due to Lemma \ref{lemma:smooth-stronglyconvex} and Lemma \ref{lemma:NS17-Eq19}, the equality is due to $\eta_2 = 1/(6\ell)$, and the last inequality is due to Lemma \ref{lemma:smooth-stronglyconvex}. This completes the proof.
\end{proof}

We now prove the following upper bound of $\bE \|y_s - y^*(x_s)\|_2^2$.

\begin{lemma}\label{lemma:upper bound ZO SGDA}
Consider \texttt{ZO-SGDA} (Algorithm \ref{ZO-SGDA}). Use the same notation and the same assumptions as in Lemma \ref{lemma:ZO-SGDA-useful-inequality}. Denote $\delta_s = \|y_s - y^*(x_s)\|_2^2$ and set $\eta_1$ as in \eqref{eta1}, and 
\be\label{gamma s}
\gamma := 1- \frac{1}{24\kappa} + 144\ell^2\kappa^3\eta_1^2\leq 1-\frac{5}{144\kappa}<1.
\ee 
It holds that
\be\label{lemma:upper bound ZO SGDA-inequality}
\bE \delta_s \leq \gamma^s \bE \delta_0 + \alpha_1 \sum_{i=0}^{s-1} \gamma^{s-1-i}\bE\|\nabla g(x_{i-1})\|_2^2 + \theta_0\sum_{i=0}^{s-1}\gamma^{s-1-i},
\ee
where 
\be\label{alpha1 s}\alpha_1 = \frac{9}{12^8\kappa(\kappa+1)^4(\ell+1)^2}, \ \theta_0 = \alpha_2 \varrho_2(\epsilon,\mu_2) + 2\varrho(\mu_2,\epsilon), \ \alpha_2 = \frac{1}{8\times 12^7\kappa(\kappa+1)^4}.
\ee
\end{lemma}

\begin{proof}
Define the filtration $\mathcal{F}_s = \Bigl\{x_s, y_s,x_{s-1},y_{s-1},...,x_{1},y_{1}\Bigr\}$. Let $\zeta_s = (\bu_{\cM_1},\xi_{\cM_1},\bu_{\cM_2},\xi_{\cM_2}),\zeta_{[s]} = (\zeta_1,\zeta_2,...,\zeta_s)$. Denote by $\bE$ taking expectation w.r.t $\zeta_{[s]}$ conditioned on $\mathcal{F}_{s}$ and then taking expectation over $\mathcal{F}_{s}$. Since $\kappa > 1$, using the Young's inequality, we have
\be\label{lemma:upper bound ZO SGDA-proof-1}\ba{ll}
\bE \delta_s & = \bE\|y^*(x_s) - y_s\|_2^2 \\
             & \leq \Bigl(1 + \frac{1}{2(12\kappa-1)}\Bigr)\bE\|y^*(x_{s-1}) - y_{s}\|_2^2 + \Bigl(1 + 2(12\kappa-1)\Bigr)\bE\|y^*(x_s) - y^*(x_{s-1})\|_2^2\\
             & \leq (\frac{24\kappa-1}{2(12\kappa-1)})(1-\frac{1}{12\kappa})\bE\|y^*(x_{s-1}) - y_s\|_2^2 + 24\kappa \bE\|y^*(x_s) - y^*(x_{s-1})\|_2^2+2\varrho(\mu_2,\epsilon)\\
             & \leq (1 - \frac{1}{24\kappa})\bE \|y^*(x_{s-1}) - y_{s-1}\|_2^2 + 24\kappa^3\bE\|x_s - x_{s-1}\|_2^2 + 2\varrho(\mu_2,\epsilon)\\
             & = (1-\frac{1}{24\kappa})\bE\delta_{s-1}+24\kappa^3\eta_1^2\bE\|G_{\mu_1}(x_{s-1},y_{s-1},\bu_{\cM_1},\xi_{\cM_1})\|_2^2+ 2\varrho(\mu_2,\epsilon) \\
             & = (1-\frac{1}{24\kappa})\bE\delta_{s-1}+\frac{\alpha_1}{6}\bE\|G_{\mu_1}(x_{s-1},y_{s-1},\bu_{\cM_1},\xi_{\cM_1})\|_2^2+ 2\varrho(\mu_2,\epsilon),
\ea\ee
where 
the second inequality is due to \eqref{lemma:ZO-SGDA-useful-inequality-inequality}, the third inequality is due to Lemma \ref{lemma:smoothed convexity}. From Lemma \ref{lemma:mini batch upper bound}, we have
\be\label{g upper bound s}\ba{ll}
& \bE_{\bu_{\cM_1},\xi_{\cM_1}} \|G_{\mu_1}(x_{s-1},y_{s-1},\bu_{\cM_1},\xi_{\cM_1})\|_2^2 \\
\leq & 3\bE\|\nabla_x f(x_{s-1},y_{s-1})\|_2^2 + \varrho_2(\epsilon,\mu_2)\\
\leq & 6\bE\|\nabla g(x_{s-1})\|_2^2 + 6\ell^2\bE\|y^*(x_{s-1}) - y_{s-1}\|_2^2 +\varrho_2(\epsilon,\mu_2),
\ea\ee
where the second inequality is due to Assumption \ref{assumption:A1-A2-A4}. Combining \eqref{lemma:upper bound ZO SGDA-proof-1} and \eqref{g upper bound s} yields \eqref{lemma:upper bound ZO SGDA-inequality} by noting \eqref{gamma s}.
\end{proof}

Similar to Lemma \ref{lemma:ZO-SGDA-useful-inequality}, we can prove the following result under the SGC assumption, i.e., Assumption \ref{SGC-assumption}.

\begin{lemma}(Linear convergence rate under SGC) \label{linear_rate_under_sgc} Under the SGC assumption (Assumption \ref{SGC-assumption}), we have:
\[
\bE \|y^*(x_{s-1})-y_s\|^2  \leq \bE \|y^*(x_{s-1})-y_{s-1}\|^2\Bigl(1-\frac{\tau}{12\ell}\Bigr) + \bar{\varrho}(\mu_2,\rho_2),
\]
where $\bar{\varrho}(\mu_2,\rho_2) = \mu_2^2d_2/3 +\frac{1}{36\ell^2}\Bigl(\frac{\mu_2^2 }{\rho_2}L^2 (d_2+6)+ \mu_2^2 \ell^2 (d_2+3)^3/2\Bigr)$  with $\eta_2 = \frac{1}{6\ell}$.
\end{lemma}
	
\begin{proof}
The proof is the almost identical to the proof of Lemma \ref{lemma:ZO-SGDA-useful-inequality}. The only difference is that we need to use Lemma \ref{lemma:SGC-zo-gradient-bound} instead of Lemma \ref{lemma:mini batch upper bound}. We omit the details for succinctness. 
\end{proof}

Now we are ready to prove Theorem \ref{ZO_SGDA_thm}.

\begin{proof}{({\bf Proof of Theorem \ref{ZO_SGDA_thm}.})}
We first prove part 1. First, the following inequalities hold:
\[\ba{ll}
     & g(x_{s+1}) \\
\leq & g(x_s) - \eta_1 \langle \nabla g(x_{s}),G_{\mu_1}(x_s,y_{s},\bu_{\cM_1},\xi_{\cM_1})\rangle + \frac{1}{2}L_g \eta_1^2 \|G_{\mu_1}(x_s,y_{s},\bu_{\cM_1},\xi_{\cM_1})\|_2^2\\
=    & g(x_s)-\eta_1\Bigl\langle \nabla_x f(x_s,y^*(x_s)) - \nabla_x f_{\mu_1}(x_s,y^*(x_s))+ \nabla_x f_{\mu_1}(x_s,y^*(x_s)) - \nabla_x f_{\mu_1}(x_s,y_{s})\\
& +\nabla_x f_{\mu_1}(x_s,y_{s}), G_{\mu_1}(x_s,y_{s},\bu_{\cM_1},\xi_{\cM_1})\Bigr\rangle + \frac{1}{2}L_g \eta_1^2 \|G_{\mu_1}(x_s,y_{s},\bu_{\cM_1},\xi_{\cM_1})\|_2^2\\
\leq & g(x_s) + \|\nabla_x f(x_s,y^*(x_s)) - \nabla_x f_{\mu_1}(x_s,y^*(x_s))\|^2/L_g + \frac{L_g\eta_1^2}{4}\|G_{\mu_1}(x_s,y_{s},\bu_{\cM_1},\xi_{\cM_1})\|^2 \\
     & + \|\nabla_x f_{\mu_1}(x_s,y^*(x_s)) - \nabla_x f_{\mu_1}(x_s,y_{s})\|^2/L_g + \frac{L_g\eta_1^2}{4}\|G_{\mu_1}(x_s,y_{s},\bu_{\cM_1},\xi_{\cM_1})\|^2 \\
     & -\eta_1\langle\nabla_x f_{\mu_1}(x_s,y_{s}), G_{\mu_1}(x_s,y_{s},\bu_{\cM_1},\xi_{\cM_1})\rangle + \frac{1}{2}L_g \eta_1^2 \|G_{\mu_1}(x_s,y_{s},\bu_{\cM_1},\xi_{\cM_1})\|_2^2 \\
\leq & g(x_s) + \frac{\ell^2}{L_g}\|y^*(x_s) - y_{s}\|_2^2  - \eta_1 \langle \nabla_x f_{\mu_1}(x_s,y_{s}), G_{\mu_1}(x_s,y_{s},\bu_{\cM_1},\xi_{\cM_1})\rangle \\
& + \eta_1^2L_g \|G_{\mu_1}(x_s,y_{s},\bu_{\cM_1},\xi_{\cM_1})\|_2^2 + \frac{\mu_1^2}{4L_g}\ell^2(d_1 +3)^{3},
\ea\]
where the first inequality is due to Lemma \ref{lemma:smoothed convexity}, the second inequality is due to Young's inequality, and the last inequality is due to Lemmas \ref{lemma:diff-f-fmu} and \ref{lemma:smooth of smoothed function}. Now take expectation with respect to $\bu_{\cM_1}$,$\xi_{\cM_1}$ to the above inequality, we get:
\be\label{proof-thm-ZO-SGDA-inequality-1}
\ba{ll}
    \eta_1 \bE \|\nabla_x f_{\mu_1}(x_s,y_{s})\|_2^2 \leq &\bE g(x_s) - \bE g(x_{s+1}) + \frac{\ell^2}{L_g}\bE\|y^*(x_s)-y_{s}\|_2^2 \\
    &+ \eta_1^2 L_g \bE\|G_{\mu_1}(x_s,y_{s},\bu_{\cM_1},\xi_{\cM_1})\|_2^2 + \frac{\mu_1^2}{4L_g}\ell^2(d_1 +3)^3.
\ea
\ee
From Lemma \ref{lemma:smooth of smoothed function}, we have
\be\label{proof-thm-ZO-SGDA-inequality-2}
    \eta_1 \bE\|\nabla_x f_{\mu_1}(x_s,y^*(x_s))\|_2^2 \leq 2\eta_1\bE\|\nabla_x f_{\mu_1}(x_s,y_s)\|_2^2 + 2\eta_1\ell^2 \|y_s - y^*(x_s)\|_2^2.
\ee
From Lemma \ref{lemma:diff-f-fmu}, we have
\be\label{proof-thm-ZO-SGDA-inequality-3}
    \eta_1\|\nabla g(x_s)\|_2^2 \leq 2\eta_1\|\nabla_x f_{\mu_1}(x_s,y^*(x_s))\|_2^2 + \frac{\eta_1 \mu_1^2}{2}\ell^2 (d_1+3)^3.
\ee
Combining \eqref{g upper bound s}, \eqref{proof-thm-ZO-SGDA-inequality-1}, \eqref{proof-thm-ZO-SGDA-inequality-2}, \eqref{proof-thm-ZO-SGDA-inequality-3} yields,
\be\label{proof-thm-ZO-SGDA-inequality-4}
\ba{ll}
     &\eta_1 \bE\|\nabla g(x_s)\|_2^2 \\
\leq &4\bE g(x_s) - 4\bE g(x_{s+1}) + \Bigl(\frac{4\ell^2}{L_g } + 4\eta_1\ell^2\Bigr)\bE\|y^*(x_s)-y_{s}\|_2^2+ \frac{\mu_1^2}{L_g}\ell^2(d_1 +3)^3 + \frac{\eta_1\mu_1^2}{2}\ell^2(d_1+3)^3\\
    & + 4\eta_1^2L_g\Bigl[6\bE\|\nabla g(x_{s})\|_2^2 + 6\ell^2\bE\|y^*(x_{s}) - y_{s}\|_2^2+\epsilon(\mu_2)\Bigr]\\
=   & 4\bE g(x_s) - 4\bE g(x_{s+1}) +24\eta_1^2L_g \bE\|\nabla g(x_{s})\|_2^2+ \theta_1\bE\delta_s + \theta_2,
\ea\ee
where 
\be\label{theta-1 s} 
\theta_1=\frac{4 \ell^2}{L_g } + 4\eta_1 \ell^2+24\eta_1^2L_g\ell^2\leq 4\ell+4\eta_1 \ell^2+24\eta_1^2\ell^3(\kappa+1),
\ee
and 
\be\label{theta-2 s}
\ba{ll}
&\theta_2 = \frac{\mu_1^2}{L_g}\ell^2(d_1 +3)^3 + \frac{\eta_1 \mu_1^2}{2}\ell^2(d_1+3)^3 + 4\eta_1^2L_g\epsilon(\mu_2)\\
\leq &\mu_1^2\ell(d_1 +3)^3 + \frac{\eta_1 \mu_1^2}{2}\ell^2(d_1+3)^3 + 4\eta_1^2(\kappa+1)\ell \epsilon(\mu_2)\\
\leq & \mu_1^2\ell(d_1 +3)^3 + \frac{\eta_1 \mu_1^2}{2}\ell^2(d_1+3)^3 + \eta_1^2(\kappa+1)\ell^3\Bigl(2\mu_1^2 (d_1+3)^3 + \frac{\mu_1^2 (d_1+6)^2 \epsilon^2}{2}\Bigr) \\ & + 2\eta_1^2 (\kappa+1)\ell \epsilon^2,
\ea
\ee
where we have used the definition of $L_g:=\ell(\kappa+1)$. Taking sum over $s=0,\ldots,S$ to both sides of \eqref{proof-thm-ZO-SGDA-inequality-4}, we get 
\be\label{proof-thm-ZO-SGDA-inequality-5}
\sum_{s=0}^S\bE \delta_s \leq \sum_{s=0}^S\gamma^s \bE \delta_0 + \alpha_1 \sum_{s=0}^S\sum_{i=0}^{s-1} \gamma^{s-1-i}\bE\|\nabla g(x_{i-1})\|_2^2 + \theta_0\sum_{s=0}^S\sum_{i=0}^{s-1}\gamma^{s-1-i}.
\ee
Moreover, from \eqref{gamma s} it is easy to obtain 
\be\label{proof-thm-ZO-SGDA-inequality-6}
\sum_{s=0}^S \gamma^s \leq 36\kappa, \quad \sum_{s=0}^S\sum_{i=0}^{s-1}\gamma^{s-1-i} \leq 36\kappa(S+1), 
\ee
and 
\be\label{proof-thm-ZO-SGDA-inequality-7}
\sum_{s=0}^S\sum_{i=0}^{s-1} \gamma^{s-1-i}\bE\|\nabla g(x_{i-1})\|_2^2 \leq 36\kappa\sum_{s=0}^S\bE\|\nabla g(x_{s})\|_2^2.
\ee
Substituting \eqref{proof-thm-ZO-SGDA-inequality-6} and \eqref{proof-thm-ZO-SGDA-inequality-7} into \eqref{proof-thm-ZO-SGDA-inequality-5}, we obtain
\be\label{proof-thm-ZO-SGDA-inequality-8}
\sum_{s=0}^S\bE \delta_s \leq 36\kappa\bE \delta_0 + 36\kappa\alpha_1 \sum_{s=0}^S\bE\|\nabla g(x_{s})\|_2^2 + 36\kappa\theta_0(S+1).
\ee
Now, summing \eqref{proof-thm-ZO-SGDA-inequality-4} over $s=0,\ldots,S$ yields
\be\label{proof-thm-ZO-SGDA-inequality-9}\ba{ll}
     &\eta_1 \sum_{s=0}^S\bE\|\nabla g(x_s)\|_2^2 \\
=    & 4\bE g(x_0) - 4\bE g(x_{S+1}) +24\eta_1^2L_g \sum_{s=0}^S\bE\|\nabla g(x_{s})\|_2^2+ \theta_1\sum_{s=0}^S\bE\delta_s + (S+1)\theta_2 \\
\leq & 4\bE g(x_0) - 4\bE g(x_{S+1}) +24\eta_1^2L_g \sum_{s=0}^S\bE\|\nabla g(x_{s})\|_2^2 \\ 
     & + \theta_1 [36\kappa\bE \delta_0 + 36\kappa\alpha_1 \sum_{s=0}^S\bE\|\nabla g(x_{s})\|_2^2 + 36\kappa\theta_0(S+1)] + (S+1)\theta_2,
\ea\ee
where the second inequality is from \eqref{proof-thm-ZO-SGDA-inequality-8}. Using \eqref{theta-1 s}, \eqref{alpha1 s} and \eqref{eta1}, it is easy to verify that 
\[36\kappa\theta_1\alpha_1 \leq \left(\frac{108}{3\times 12^3} + \frac{108}{12^7}+\frac{54}{4\times 12^{10}}\right)\eta_1 \leq 0.021\eta_1,\]
which together with $L_g:=(\kappa+1)\ell$ yields 
\be\label{coefficient-1 s}
36\kappa\theta_1\alpha_1 + 24\eta_1^2L_g \leq 0.021\eta_1 + 0.0003\eta_1 = 0.0213\eta_1.
\ee
Combining \eqref{proof-thm-ZO-SGDA-inequality-9} and \eqref{coefficient-1 s} yields 
\be\label{proof-thm-ZO-SGDA-inequality-10}\ba{ll}
     & 0.9787 \eta_1 \sum_{s=0}^S\bE\|\nabla g(x_s)\|_2^2 \\
\leq & 4\bE g(x_0) - 4\bE g(x_{S+1}) + \theta_1 [36\kappa\bE \delta_0 + 36\kappa\theta_0(S+1)] + (S+1)\theta_2.
\ea\ee
Dividing both sides of \eqref{proof-thm-ZO-SGDA-inequality-10} by $0.9787 \eta_1(S+1)$ yields 
\be\label{proof-thm-ZO-SGDA-inequality-11}\ba{ll}
 \frac{1}{S+1} \sum_{s=0}^S\bE\|\nabla g(x_s)\|_2^2 \leq  \frac{4\Delta_g}{0.9787 \eta_1(S+1)} + \frac{36\kappa\theta_1 \bE \delta_0}{0.9787 \eta_1(S+1)} + \frac{36\kappa\theta_1 \theta_0}{0.9787 \eta_1} + \frac{\theta_2}{0.9787 \eta_1},
\ea\ee
where $\Delta_g := g(x_0)-\min_{x\in\br^{d_1}} g(x)$. Now we only need to upper bound the right hand side of \eqref{proof-thm-ZO-SGDA-inequality-11} by $O(\epsilon^2)$. Note that by the choice of parameters in \eqref{thm3.1-S-mu1-mu2}, the right hand side of \eqref{proof-thm-ZO-SGDA-inequality-11} is $O(\epsilon^2) + O(\epsilon^4)$. Hence, with $\epsilon \in (0,1)$, we get the required result. This completes the proof of the part 1 of Theorem \ref{ZO_SGDA_thm}.

We now prove part 2. Denote $\delta_s = \|y_s - y^*(x_s)\|_2^2$ and set $\eta_1$ as in \eqref{eta1}, and $\gamma$ is defined as in \eqref{gamma s}.

From Lemma \ref{linear_rate_under_sgc} we have:
\[
\bE \|y^*(x_{s-1})-y_s\|^2  \leq \bE \|y^*(x_{s-1})-y_{s-1}\|^2\Bigl(1-\frac{\tau}{12\ell}\Bigr) + \bar{\varrho}(\mu_2,\rho_2).
\]
Using Young's inequality on $\delta_s$, we have:
\[
\bE \delta_s \leq \Bigl(1-\frac{1}{24\kappa}\Bigr)\bE\delta_{s-1}+\frac{\alpha_1}{6}\bE\|G_{\mu_1}(x_{s-1},y_{s-1},\bu_{\cM_1},\xi_{\cM_1})\|_2^2+ 2\bar{\varrho}(\mu_2,\rho_2).
\]
Following the same way for proving \eqref{proof-thm-ZO-SGDA-inequality-8}, it is easy to show that 
\[
\bE \delta_s \leq \gamma^s \bE \delta_0 + \alpha_1 \sum_{i=0}^{s-1} \gamma^{s-1-i}\bE\|\nabla g(x_{i-1})\|_2^2 + \theta_0\sum_{i=0}^{s-1}\gamma^{s-1-i},
\]
in which
\be\label{alpha1 s}\alpha_1 = \frac{9}{12^8\kappa(\kappa+1)^4(\ell+1)^2}, \ \bar{\theta}_0 = \alpha_2 \bar{\varrho}_2(\mu_2,\rho_2) + 2\bar{\varrho}(\mu_2,\rho_2), \ \alpha_2 = \frac{1}{8\times 12^7\kappa(\kappa+1)^4}.
\ee
Using the above expressions and following the result of \eqref{proof-thm-ZO-SGDA-inequality-4}, we have:
	\be\label{SGC-temp-1}
     0.9787 \eta_1 \sum_{s=0}^S\bE\|\nabla g(x_s)\|_2^2 \leq 4\bE g(x_0) - 4\bE g(x_{S+1}) + \bar{\theta}_1 [36\kappa\bE \delta_0 + 36\kappa\bar{\theta}_0(S+1)] + (S+1)\bar{\theta}_2,
	\ee
	with 
\[
\ba{ll}
		\bar{\theta}_1 &= \Bigl(4\ell^2/L_g + 4\eta_1 \ell^2 + 24\eta_1^2 L_g \ell^2 \Bigr)\\
		\bar{\theta}_2 &= \frac{\mu_1^2}{L_g}\ell^2(d_1 +3)^3 + \frac{\eta_1\mu_1^2}{2}\ell^2(d_1+3)^3 + 4\eta_1L_g \bar{\varrho}_1(\mu_1,\rho_1).
		\ea
\]
	Divide both sides of \eqref{SGC-temp-1} by $0.9787\eta_1(S+1)$, we get
		\be\label{proof_SGDA_final_bound-0}
 \frac{1}{S+1} \sum_{s=0}^S\bE\|\nabla g(x_s)\|_2^2 \leq  \frac{4\Delta_g}{0.9787 \eta_1(S+1)} + \frac{36\kappa\bar{\theta}_1 \bE \delta_0}{0.9787 \eta_1(S+1)} + \frac{36\kappa\bar{\theta}_1 \bar{\theta}_0}{0.9787 \eta_1} + \frac{\bar{\theta}_2}{0.9787 \eta_1}.
\ee
{According to Remark \ref{remark:cversusuc}, we know that $\bE \delta_0$ is upper bounded by a constant.}
Choosing $\mu_1 = \cO (\min(1,\rho_1)\ell (d_1+3)^{3/2} )$, $\mu_2 = \cO (\min(1,\rho_2)\ell (d_2+3)^{3/2} )$ , we guarantee that the right hand side of \eqref{proof_SGDA_final_bound-0} is upper bounded by $O(\epsilon^2) + O(\epsilon^4)$. Under Assumption \ref{SGC-assumption}, since we choose $|\cM_1| =\cO(\rho_1 d_1), |\cM_2| = \cO(\rho_1 d_2) $ the total number of calls to stochastic zeroth-order oracle is $\cO\Bigl(\kappa^5 (d_1\rho_1+d_2\rho_2)\epsilon^{-2}\Bigr)$. This completes the proof of part 2.

\end{proof}

\section{Convergence analysis of \texttt{ZO-SGDMSA} (Algorithm \ref{alg:ZO-SGDMSA})}

First, we show the following iteration complexity of the inner loop for $y$ in Algorithm \ref{alg:ZO-SGDMSA}.
\begin{lemma}\label{GDmax_linear_convergence}
In Algorithm \ref{alg:ZO-SGDMSA}, setting $\eta_2=1/(6\ell)$, $\mu_2=\cO(\kappa^{-1/2}d_2^{-3/2}\epsilon)$ and $T= \mathcal{O}(\kappa  \log(\epsilon^{-1}))$. For fixed $x_s$ in the $s$-th iteration, we have $\bE\|y^*(x_s)-y_{T}(x_s)\|_2^2 \leq \epsilon^2$.
\end{lemma}

\begin{proof}
According to the updates in Algorithm \ref{alg:ZO-SGDMSA}, we have
\[\ba{ll}
  & \|y^*(x_s)-y_{t+1}(x_s)\|^2 \\
= &(\|\Proj_{\mathcal{Y}}(y_t(x_s) + \eta_2 H_{\mu_2}(x_s,y_{t}(x_s),\bu_{\cM_2},\xi_{\cM_2}) - y^*(x_s))\|_2^2)\\
\leq & \|y^*(x_s)-y_t(x_s)\|^2 +2\eta_2\langle H_{\mu_2}(x_s, y_{t}(x_s),\bu_{\cM_2},\xi_{\cM_2}),y_t(x_s) - y^*(x_s)\rangle + \eta_2^2\|H_{\mu_2}(x_s, y_{t}(x_s),\bu_{\cM_2},\xi_{\cM_2}\|_2^2.
\ea\]
For a given $s$, denote by $\bE$ taking expectation with respect to random samples $\bu_{\cM_2},\xi_{\cM_2}$ conditioned on all previous iterations. By taking expectation to both sides of this inequality, we obtain
\[\ba{ll}
     & \bE \|y^*(x_s)-y_{t+1}(x_s)\|^2 \\
\leq & \bE \|y^*(x_s)-y_t(x_s)\|^2 - 2\eta_2 \langle -\nabla_y f_{\mu_2}(x_s, y_{t}(x_s)),y_t(x_s) - y^*(x_s)\rangle + \eta_2^2 \bE\|H_{\mu_2}(x_s, y_{t}(x_s),\bu_{\cM_1},\xi_{\cM_1})\|_2^2\\
\leq & \bE \|y^*(x_s)-y_t(x_s)\|^2 - 2\eta_2 \langle -\nabla_y f_{\mu_2}(x_s, y_{t}(x_s)),y_t(x_s) - y^*(x_s)\rangle + \eta_2^2 (3\|\nabla_y f(x_s,y_t(x_s))\|_2^2 + \varrho_2(\epsilon,\mu_2)\\
\leq & \bE \|y^*(x_s)-y_t(x_s)\|^2 - 2\eta_2[f_{\mu_2}(x_s,y^*(x_s)) - f_{\mu_2}(x_s,y_t(x_s))] + \eta_2^2 (3\|\nabla_y f(x_s,y_t(x_s))\|_2^2 + \varrho_2(\epsilon,\mu_2) ) \\
\leq & \bE \|y^*(x_s)-y_t(x_s)\|^2 - 2\eta_2(f(x_s,y^*(x_s)) - f(x_s,y_{t}(x_s))) + 2\mu_2^2 d_2 \eta_2 \ell+ \eta_2^2(6L_2(f(x_s,y^*(x_s))-f(x_s,y_{t}(x_s)))\\
&+\eta_2^2\varrho_2(\epsilon,\mu_2)\\
= & \bE \|y^*(x_s)-y_t(x_s)\|^2 - (f(x_s,y^*(x_s))-f(x_s,y_{t}(x_s)))/(6\ell) \\
& + \epsilon^2/(72\ell^2) + \mu_2^2(d_2+3)^3/72 + \mu_2^2 (d_2+6)^2 \epsilon^2/288\\
\leq & \bE \|y^*(x_s)-y_t(x_s)\|^2 \Bigl(1-\frac{\tau}{12\ell}\Bigr) + \epsilon^2/(72\ell^2) + \mu_2^2(d_2+3)^3/72 + \mu_2^2 (d_2+6)^2 \epsilon^2/288,
\ea\]
where the second inequality is due to Lemma \ref{lemma:mini batch upper bound}, the third inequality is due to the concavity of $f_{\mu_2}(x_s,\cdot)$ (see Lemma \ref{lemma:f-mu-convex}), the fourth inequality is due to Lemmas \ref{lemma:NS17-Eq19} and \ref{lemma:smooth-stronglyconvex}, the equality is due to $\eta_2=1/(6\ell)$, and the last inequality is due to Lemma \ref{lemma:smooth-stronglyconvex}.

Define $\delta = 12\ell( \epsilon^2/(72\ell^2) + \mu_2^2(d_2+3)^3/72 + \mu_2^2 (d_2+6)^2 \epsilon^2/288)/\tau$. From the above inequality, we have
\[\ba{ll}
\bE \|y^*(x_s)-y_t(x_s)\|^2-\delta&\leq (\bE \|y^*(x_s)-y_{t-1}(x_s)\|^2-\delta)\Bigl(1-\frac{\tau}{12\ell}\Bigr) \\
& \leq(\bE \|y^*(x_s)-y_0(x_s)\|^2 - \delta)\Bigl(1-\frac{\tau}{12\ell}\Bigr)^{t}\\
& \leq \bE \|y^*(x_s)-y_0(x_s)\|^2\Bigl(1-\frac{\tau}{12\ell}\Bigr)^{t} \leq D^2 \Bigl(1 - \frac{\tau}{12\ell}\Bigr)^t,
\ea\]
where the last inequality is due to Assumption \ref{assumption:A1-A2-A4}. Now it is clear that in order to ensure that $\bE \|y^*(x_s)-y_T(x_s)\|^2 \leq \epsilon^2$, we need $T = \cO(\kappa  \log(\epsilon^{-1}))$ and $\mu_2=\cO(\kappa^{-1/2}d_2^{-3/2}\epsilon)$.
\end{proof}

We are now ready to prove Theorem \ref{ZO_SGDmax_thm}.

\begin{proof}{({\bf Proof of Theorem \ref{ZO_SGDmax_thm}.})}
We first prove Part 1. First, the following inequalities hold:
\[\ba{ll}
     & g(x_{s+1}) \\
\leq & g(x_s)-\eta_1\langle \nabla_x g(x_{s}),G_{\mu_1}(x_s,y_{s+1},\bu_{\cM_1},\xi_{\cM_1})\rangle +\frac{1}{2}L_g\eta_1^2\|G_{\mu_1}(x_s,y_{s+1},\bu_{\cM_1},\xi_{\cM_1})\|_2^2\\
= & g(x_s)-\eta_1\Bigl\langle \nabla_x f(x_s,y^*(x_s)) - \nabla_x f_{\mu_1}(x_s,y^*(x_s))+ \nabla_x f_{\mu_1}(x_s,y^*(x_s)) - \nabla_x f_{\mu_1}(x_s,y_{s+1})\\ &~+\nabla_x f_{\mu_1}(x_s,y_{s+1}),G_{\mu_1}(x_s,y_{s+1},\bu_{\cM_1},\xi_{\cM_1})\Bigr\rangle + \frac{1}{2}L_g \eta_1^2 \|G_{\mu_1}(x_s,y_{s+1},\bu_{\cM_1},\xi_{\cM_1})\|_2^2\\
\leq & g(x_s) + \|\nabla_x f(x_s,y^*(x_s)) - \nabla_x f_{\mu_1}(x_s,y^*(x_s))\|^2/L_g + \frac{L_g\eta_1^2}{4}\|G_{\mu_1}(x_s,y_{s+1},\bu_{\cM_1},\xi_{\cM_1})\|^2 \\
     & + \|\nabla_x f_{\mu_1}(x_s,y^*(x_s)) - \nabla_x f_{\mu_1}(x_s,y_{s+1})\|^2/L_g + \frac{L_g\eta_1^2}{4}\|G_{\mu_1}(x_s,y_{s+1},\bu_{\cM_1},\xi_{\cM_1})\|^2 \\
     & -\eta_1\langle\nabla_x f_{\mu_1}(x_s,y_{s+1}), G_{\mu_1}(x_s,y_{s+1},\bu_{\cM_1},\xi_{\cM_1})\rangle + \frac{1}{2}L_g \eta_1^2 \|G_{\mu_1}(x_s,y_{s+1},\bu_{\cM_1},\xi_{\cM_1})\|_2^2 \\
\leq & g(x_s) + \frac{\ell^2}{L_g}\|y^*(x_s) - y_{s+1}\|_2^2  - \eta_1 \langle \nabla_x f_{\mu_1}(x_s,y_{s+1}), G_{\mu_1}(x_s,y_{s+1},\bu_{\cM_1},\xi_{\cM_1})\rangle \\
& + \eta_1^2L_g \|G_{\mu_1}(x_s,y_{s+1},\bu_{\cM_1},\xi_{\cM_1})\|_2^2 + \frac{\mu_1^2}{4L_g}\ell^2(d_1 +3)^{3},
\ea\]
where the first inequality is due to Lemma \ref{lemma:smoothed convexity}, the second inequality is due to Young's inequality, and the last inequality is due to Lemmas \ref{lemma:diff-f-fmu} and \ref{lemma:smooth of smoothed function}. Now take expectation with respect to $\bu_{\cM_1},\xi_{\cM_1}$ to the above inequality, we get:
\be\label{proof-thm-ZO-SGDMSA-inequality-1}\ba{ll}
& \eta_1 \bE\|\nabla_x f_{\mu_1}(x_s,y_{s+1})\|_2^2 \\
\leq & \bE g(x_s) -\bE g(x_{s+1}) + \frac{\ell^2}{L_g }\bE\|y^*(x_s)-y_{s+1}\|_2^2 + \eta_1^2 L_g \bE\|G_{\mu_1}(x_s,y_{s+1},\bu_{\cM_1},\xi_{\cM_1})\|_2^2 + \frac{\mu_1^2}{4L_g}\ell^2(d_1 +3)^3 \\
\leq & \bE g(x_s) -\bE g(x_{s+1}) + \frac{\ell^2}{L_g }\bE\|y^*(x_s)-y_{s+1}\|_2^2 + \eta_1^2 L_g \Bigl(3\|\nabla_x f(x_s,y_{s+1})\|_2^2 + \varrho_1(\epsilon,\mu_1)\Bigr)+ \frac{\mu_1^2}{4L_g}\ell^2(d_1 +3)^3,
\ea\ee
where the second inequality is due to Lemma \ref{lemma:mini batch upper bound}. From Lemma \ref{lemma:diff-f-fmu} we have
\be\label{proof-thm-ZO-SGDMSA-inequality-2}\ba{ll}
\bE\|\nabla_x f(x_s,y_{s+1})\|_2^2 \leq 2\bE\|\nabla_x f_{\mu_1}(x_s,y_{s+1})\|_2^2 + \mu_1^2\ell^2(d_1+3)^3/2.
\ea\ee
Combining \eqref{proof-thm-ZO-SGDMSA-inequality-1} and \eqref{proof-thm-ZO-SGDMSA-inequality-2}, and noting $\eta_1=1/(12L_g)$, we have
\be\label{proof-thm-ZO-SGDMSA-inequality-3}\ba{ll}
\bE\|\nabla_x f(x_s,y_{s+1})\|_2^2\leq & 48L_g \Bigl[\bE g(x_s) -\bE g(x_{s+1})\Bigr] + 48 \ell^2\bE\|y^*(x_s)-y_{s+1}\|_2^2\\
&+ 13\mu_1^2 \ell^2 (d_1+3)^3 + \varrho_1(\epsilon,\mu_1)/12.
\ea\ee
It then follows that
\be\label{proof-thm-ZO-SGDMSA-inequality-4}\ba{ll}
& \bE\|\nabla g(x_s)\|_2^2 \\
\leq & 2 \bE\|\nabla_x g(x_s) - \nabla_x f(x_s,y_{s+1})\|_2^2 + 2\bE\|\nabla_x f(x_s,y_{s+1})\|_2^2 \\
\leq & 2\ell^2 \bE\|y^*(x_s) - y_{s+1}\|_2^2 + 2\bE\|\nabla_x f(x_s,y_{s+1})\|_2^2 \\
\leq & 96L_g \Bigl[\bE g(x_s) -\bE g(x_{s+1})\Bigr] + 98\ell^2\bE\|y^*(x_s)-y_{s+1}\|_2^2\\
&+ 26\mu_1^2 \ell^2 (d_1+3)^3+\varrho_1(\epsilon,\mu_1)/6,
\ea\ee
where the second inequality is due to Assumption \ref{assumption:A1-A2-A4}, and the last inequality is due to \eqref{proof-thm-ZO-SGDMSA-inequality-3}.

Take the sum over $s=0,\ldots,S$ to both sides of \eqref{proof-thm-ZO-SGDMSA-inequality-4}, we get
\be\label{proof-thm-ZO-SGDMSA-inequality-5}\ba{ll}
\frac{1}{S+1}\sum_{s=0}^S \bE \|\nabla g(x_s)\|_2^2 \leq & \frac{96L_g}{S+1}\bE [g(x_0) - g(x_{S+1})] + \frac{98\ell^2}{S+1}\sum_{s=0}^S\bE\|y^*(x_s)-y_{s+1}\|_2^2\\
&+ 26\mu_1^2 \ell^2(d_1+3)^3+\varrho_1(\epsilon,\mu_1)/6.
\ea\ee
Denote $\Delta_g = g(x_0) - \min_{x\in\mathbb{R}^{d_1}}(g(x))$. From Lemma \ref{GDmax_linear_convergence}, we know that when $T = \cO (\kappa \log(\epsilon^{-1}))$, we have $\bE\|y^*(x_s)-y_{s+1}\|^2\leq\epsilon^2$ (note that $y_{s+1}=y_T(x_s)$). Therefore, choosing
parameters as in \eqref{thm3.2-S-mu1-mu2} guarantees that the right hand side of \eqref{proof-thm-ZO-SGDMSA-inequality-5} is upper bounded by $O(\epsilon^2) + O(\epsilon^4)$. Hence, with $\epsilon \in (0,1)$, we get the required result and thus an $\epsilon$-stationary point is found.
This completes the proof of Part 1.

We next prove Part 2. From Lemma \ref{linear_rate_under_sgc} we have 
\[
\bE \|y^*(x_{s-1})-y_s\|^2  \leq \bE \|y^*(x_{s-1})-y_{s-1}\|^2\Bigl(1-\frac{\tau}{12\ell}\Bigr) + \bar{\varrho}(\mu_2,\rho_2).
\]
Choosing $\delta = \frac{12\ell}{\tau}\bar{\varrho}(\mu_2,\rho_2)$, we have:
\[
\ba{ll}
\bE \|y^*(x_s)-y_t(x_s)\|^2-\delta&\leq (\bE \|y^*(x_s)-y_{t-1}(x_s)\|^2-\delta)\Bigl(1-\frac{\tau}{12\ell}\Bigr) \\
& \leq(\bE \|y^*(x_s)-y_0(x_s)\|^2 - \delta)\Bigl(1-\frac{\tau}{12\ell}\Bigr)^{t}\\
& \leq \bE \|y^*(x_s)-y_0(x_s)\|^2\Bigl(1-\frac{\tau}{12\ell}\Bigr)^{t} \leq D^2 \Bigl(1 - \frac{\tau}{12\ell}\Bigr)^t.
\ea
\]
In order to ensure that 
$\bE \|y^*(x_s)-y_T(x_s)\|^2 \leq \epsilon^2$, we need $T = \cO(\kappa  \log(\epsilon^{-1}))$ and $\mu_2=\cO(\min(1,\rho_2)\kappa^{-1/2}d_2^{-3/2}\epsilon)$. From \eqref{proof-thm-ZO-SGDMSA-inequality-4} and \eqref{ZO_GD_SGC_bound} we have
\be\label{SGC-temp-2}
\ba{ll}
\bE\|\nabla g(x_s)\|_2^2 &\leq 96L_g \Bigl[\bE g(x_s) -\bE g(x_{s+1})\Bigr] + 98\ell^2\bE\|y^*(x_s)-y_{s+1}\|_2^2\\
&+ 26\mu_1^2 \ell^2 (d_1+3)^3+\varrho_1(\mu_1,\rho_1)/6.
\ea
\ee
Taking the sum over $s = 0,...,S$ to both sides of \eqref{SGC-temp-2}, we get:
\be\ba{ll}\label{proof_ZO_SGDMSA_upper}
\frac{1}{S+1}\sum_{s=0}^S \bE \|\nabla g(x_s)\|_2^2 \leq & \frac{96L_g}{S+1}\bE [g(x_0) - g(x_{S+1})] + \frac{98\ell^2}{S+1}\sum_{s=0}^S\bE\|y^*(x_s)-y_{s+1}\|_2^2\\
&+ 26\mu_1^2 \ell^2(d_1+3)^3+\varrho_1(\mu_1,\rho_1)/6.
\ea\ee
Recall that $\varrho_1(\mu_1,\rho_1) = \frac{\mu_1^2 }{\rho_1}\ell^2 (d_1+6)+ \mu_1^2 \ell^2 (d_1+3)^3/2$, choosing $\mu_1 = \cO\Bigl(\min(1,\rho_1)\ell (d_1)^{-3/2}\Bigr)
$, we guarantee that the right hand side of \eqref{proof_ZO_SGDMSA_upper} is upper bounded by $O(\epsilon^2) + O(\epsilon^4)$. Hence, with $\epsilon \in (0,1)$, we get the required result and thus an $\epsilon$-stationary point is found. This completes the proof of Part 2.

\end{proof}

\end{document}